\documentclass{article}
\usepackage[top=1in, bottom=1in, left=1in, right=1in]{geometry}

\usepackage[colorinlistoftodos]{todonotes}

\usepackage{hyperref}
\usepackage{natbib}
\usepackage{booktabs}
\usepackage{mathtools}
\usepackage{subcaption}
\usepackage{comment}
\usepackage{graphicx} 
\usepackage{amsmath,amssymb}
\usepackage{float}
\usepackage{authblk}
\usepackage{amsthm}
\usepackage{amsthm}
\usepackage{tikz-cd}
\usepackage{verbatim}
\usepackage{xcolor}
\usepackage[colorinlistoftodos]{todonotes}
\newtheorem{theorem}{Theorem}
\newtheorem{lemma}{Lemma}
\usepackage{graphicx} 

\title{From Symmetry to Invariance: Learning Galois Equivalent Representations in Finite Fields}
\author{
Zheng Zhang \thanks{Corresponding author: Department of Mathematics, Towson University, 7800 York Rd, Towson, MD 21204, USA. Email: \texttt{zzhang@towson.edu}.},
Na Zhang \thanks{Department of Mathematics, Towson University, 7800 York Rd, Towson, MD 21204, USA. Email: \texttt{nzhang@towson.edu}.}
}
\date{}
\date{}

\begin{document}

\maketitle

\begin{abstract}
Neural networks can learn algebraic operations from finite examples,
but it remains unclear whether this ability transfers across
mathematically equivalent representations of the same operation. We
study this question through multiplication in finite fields under
changes of basis. The Galois action organizes basis representations
into orbits, and bases in the same orbit induce the same coordinate
multiplication map. This structure allows us to separate learning
multiplication from transferring it to basis representations that are
not used for training. We examine several ways of providing or
recovering the relevant orbit structure, including invariant labels,
basis matrices, orbit recognition, and algebraic decomposition. Our
main approach trains a model to predict the Galois action between basis
representations. Repeated applications of the learned transformation
are then used to construct a canonical representative for each orbit,
which supports multiplication on held-out bases through exact
canonical matching. This provides a concrete mechanism for converting
a learned algebraic symmetry into an invariant representation that can
be used for transfer.
\end{abstract}

\section{Introduction}
Neural networks have demonstrated the ability to learn discrete
mathematical operations, including modular arithmetic and finite group
composition
\cite{power2022grokking,Nanda2023ProgressMF,
pmlr-v235-stander24a,kvinge2026can}.
More recently, this line of work has been extended from group operations
to multiplication in more general finite dimensional algebras
\cite{notsawo2026grokking}.
A central question in this line of work is whether successful
generalization reflects the learning of underlying mathematical
structure or merely the memorization of a finite table of examples.

\subsection{From Grokking to Algebraic Representation Learning}
\label{sec:grokking-related}
Grokking provides a useful setting for studying the relationship
between generalization and learned mathematical structure. The original
experiments of \citet{power2022grokking} showed that neural networks
trained on modular arithmetic can achieve almost perfect training
accuracy while remaining in a memorizing regime, and begin to
generalize to the test set only much later.

Subsequent work has related this delayed generalization to optimization
dynamics, implicit bias, feature learning, and the emergence of
structured internal representations
\cite{liu2022towards,lyu2024dichotomy,gromov2023grokking,
kumar2024grokking,rubin2024grokking,mallinar2025emergence}.
Mechanistic analyses have also identified changes in the computations
implemented by the model during training
\cite{Nanda2023ProgressMF}. For our purposes, the main implication of
this literature is that generalization in algebraic tasks can be
studied not only through test accuracy, but also through the
representations and computational structure learned by the model.

This perspective is particularly explicit in studies of modular
arithmetic. \citet{Nanda2023ProgressMF} identified Fourier representations and the
computational mechanism used by Transformers to perform modular
addition. Moreover, \citet{gromov2023grokking} connected grokking to the
emergence of feature representations determined by the underlying
arithmetic task. Related studies by \cite{doshi2024togrok} and \cite{mallinar2025emergence}
have further investigated how memorizing and generalizing solutions
differ in their learned representations.
These studies suggest that successful generalization can
coincide with the emergence of representations that reflect structure
in the underlying algebraic task.

This line of research has also expanded beyond modular arithmetic.
\cite{pmlr-v235-stander24a} studied multiplication in permutation groups
and identified learned mechanisms exploiting subgroup and coset
structure. In addition, \citet{kvinge2026can} investigated which
broader group theoretic properties are encoded by models trained to
predict group operations. More recently, \citet{notsawo2026grokking} extended the study of
grokking from group operations to multiplication in
finite dimensional algebras. They relate generalization to algebraic
properties of the multiplication law, structural properties of the
associated structure tensor, and learned embeddings aligned with the
algebra's representation. Their experiments evaluate generalization
while keeping the coordinate presentation of the algebra fixed between
training and test. Overall, these results provide growing evidence
that models trained on finite algebraic operations can learn structure
beyond individual input and output pairs.

This raises  a related question. If a model learns an algebraic
operation in one coordinate presentation, does that knowledge transfer
when the same operation is expressed in a different but mathematically
equivalent presentation? We study this question using multiplication in finite fields under changes of basis. Whereas the preceding work focuses on learning algebraic structure within a fixed representation,
our focus is on recognizing and exploiting equivalences across
different representations of the same operation.

\subsection{Generalization across Equivalent Basis Representations}
\label{sec:intro-basis-generalization}
Multiplication in finite fields provides a concrete setting in which
equivalent representations can be characterized exactly. Let
\(F=\mathbb F_{2^n}\). Each ordered basis \(B\) of \(F\) over
\(\mathbb F_2\) induces a coordinate multiplication map
\[
m_B([x]_B,[y]_B)=[xy]_B,
\]
while the underlying field multiplication remains fixed. The
Frobenius automorphism
\[
\sigma(x)=x^2
\]
acts componentwise on ordered bases and partitions them into Galois
orbits. Two bases induce the same coordinate multiplication map if and
only if they belong to the same Galois orbit
(see Section~\ref{sec:mathematical-framework}).

Our main experiments focus on $\mathbb F_8$. Its Galois orbit
structure allows us to study whether a model treats mathematically
equivalent basis representations consistently. We examine how this
behavior depends on the information provided to the model and the
objective used during training. Finally, we show that a learned Galois
action can be used to construct canonical representatives that support
multiplication across equivalent bases.

\subsection{Our Contributions}
\label{sec:contributions}

Our main contributions are as follows.

\begin{itemize}

\item \textbf{A controlled setting based on equivalent coordinate
representations.}
We introduce a setting in which the coordinate representation of an
algebraic operation changes but the underlying operation remains fixed.
We use the Galois action on ordered bases to characterize equivalent
representations and to study whether multiplication learned in one
representation can be transferred to another. 

\item \textbf{A constructive route from Galois symmetry to
invariance.}
We train a model to predict the Frobenius action between basis
representations and study its behavior under repeated application. We
then combine the learned transformation with a fixed canonicalization
procedure to construct a common representation for bases in the same
orbit. In \(\mathbb F_8\), this representation supports near perfect
multiplication on held-out bases under exact canonical matching.

\item \textbf{An empirical study of the conditions supporting
algebraic transfer.}
Through controlled experiments in \(\mathbb F_8\) and a fixed sample
of Galois orbits in \(\mathbb F_{16}\), we study several ways of
presenting algebraic structure to the model. The experiments compare relational recognition with direct
identification and examine the
relationship between the accuracy of the learned transformation and
performance in composition, canonicalization, and downstream
multiplication.

\end{itemize}

\section{Mathematical Framework}
\label{sec:mathematical-framework}

We focus on multiplication in the finite field
\[
F=\mathbb F_{2^n}.
\]
The field \(F\) is an \(n\) dimensional vector space over
\(\mathbb F_2\), so its elements can be represented by coordinate
vectors relative to different bases. The multiplication operation in
\(F\) is fixed, but its coordinate representation depends on the
chosen basis. We first fix a polynomial representation of \(F\) and
its associated basis, and then use it as a reference for describing
all other ordered bases.

\subsection{Multiplication in Different Bases}
\label{sec:basis-multiplication}

A standard way to construct the finite field \(F\) is
\[
F
=
\mathbb F_2[\alpha]/\bigl(p(\alpha)\bigr),
\]
where
\[
p(\alpha)\in\mathbb F_2[\alpha]
\]
is irreducible of degree \(n\). Every element \(x\in F\) can be
written uniquely as
\[
x
=
x_0+x_1\alpha+\cdots+x_{n-1}\alpha^{n-1},
\qquad
x_i\in\mathbb F_2.
\]
Therefore,
\[
E=(1,\alpha,\ldots,\alpha^{n-1})
\]
is a basis of \(F\) over \(\mathbb F_2\). We choose this particular
basis as the reference for describing all other bases.

Let
\[
V=\mathbb F_2^n.
\]
The coordinates of \(x\) relative to \(E\) are
\[
[x]_E
=
(x_0,\ldots,x_{n-1})
\in V.
\]

More generally, let
\[
B=(b_0,\ldots,b_{n-1})
\]
be any ordered basis of \(F\) over \(\mathbb F_2\). Every \(x\in F\) has a unique
expression
\[
x=u_0b_0+\cdots+u_{n-1}b_{n-1},
\qquad
u_i\in\mathbb F_2,
\]
and we write
\[
[x]_B
=
(u_0,\ldots,u_{n-1})
\in V
\]
for its coordinate vector relative to \(B\).

Associated with \(B\) is the coordinate multiplication map
\[
m_B:V\times V\longrightarrow V,
\]
defined by
\begin{equation}
m_B([x]_B,[y]_B)=[xy]_B,
\qquad
x,y\in F.
\label{eq:mulmap}
\end{equation}
The multiplication operation in \(F\) is fixed, while its coordinate
representation depends on the chosen basis.

To relate the coordinates associated with \(B\) to those associated
with the reference basis \(E\), define
\[
P_B
=
\begin{bmatrix}
\label{matrix:PB}
[b_0]_E & [b_1]_E & \cdots & [b_{n-1}]_E
\end{bmatrix}
\in\operatorname{GL}_n(\mathbb F_2).
\]
The columns of \(P_B\) are the coordinates of the elements of \(B\)
relative to \(E\). Consequently,
\begin{equation}
[x]_E=P_B[x]_B,
\label{eq:coordinate-conversion}
\end{equation}
for every \(x\in F\). Therefore, \(P_B\) converts coordinates relative to
\(B\) into coordinates relative to \(E\), while \(P_B^{-1}\) performs
the reverse conversion.

Let \(m_E\) denote the coordinate multiplication map associated with
the reference basis \(E\). For
\[
u=[x]_B,
\qquad
v=[y]_B,
\]
we have
\[
P_Bu=[x]_E,
\qquad
P_Bv=[y]_E.
\]
Therefore,
\[
m_E(P_Bu,P_Bv)
=
m_E([x]_E,[y]_E)
=
[xy]_E.
\]
Converting the result back into coordinates relative to \(B\) gives
\begin{equation}
m_B(u,v)
=
P_B^{-1}m_E(P_Bu,P_Bv).
\label{eq:basis-decomposition}
\end{equation}
This means that multiplication relative to \(B\) can be carried out by
converting the operands to the reference coordinates, multiplying
there, and converting the result back to the original coordinates.

Since every ordered basis \(B\) determines a unique matrix
\[
P_B\in\operatorname{GL}_n(\mathbb F_2),
\]
and every matrix in \(\operatorname{GL}_n(\mathbb F_2)\) determines
an ordered basis of \(F\), we can compute the number of ordered bases by the following formula
\[
\left|\operatorname{GL}_n(\mathbb F_2)\right|
=
\prod_{i=0}^{n-1}(2^n-2^i).
\]
In particular,
\[
\left|\operatorname{GL}_3(\mathbb F_2)\right|=168,
\]
so \(\mathbb F_8\) has \(168\) ordered bases over \(\mathbb F_2\).

\subsection{Galois Orbits and Equivalent Representations}
\label{sec:Galois orbits}

Not all ordered bases induce distinct coordinate multiplication maps.
The relevant equivalence is determined by the Galois action on $F$.
The Galois group
\[
\operatorname{Gal}(F/\mathbb{F}_2)
=
\langle\sigma\rangle
\]
is generated by the Frobenius automorphism
\[
\sigma(x)=x^2.
\]
It acts componentwise on an ordered basis:
\[
\sigma(B)
=
(b_0^2,\ldots,b_{n-1}^2).
\]
The resulting Galois orbit is
\[
\mathcal{O}_B
=
\{B,\sigma(B),\ldots,\sigma^{n-1}(B)\}.
\]

The Frobenius action on ordered bases is free, so every Galois orbit
contains exactly $n$ bases. We record this fact explicitly for later
use.

\begin{lemma}
\label{lem:free-galois-action}
Let $F=\mathbb{F}_{2^n}$ and let
$\operatorname{Gal}(F/\mathbb{F}_2)=\langle\sigma\rangle$ act
componentwise on the ordered $\mathbb{F}_2$-bases of $F$. Then
\[
\sigma^k(B)=B
\]
for an ordered basis $B$ if and only if
\[
k\equiv 0\pmod n.
\]
Consequently, every Galois orbit of ordered bases has exactly $n$
elements.
\end{lemma}

\begin{proof}
See Appendix~\ref{app:proofs}.
\end{proof}

More importantly, Galois orbits characterize exactly when two bases
induce the same coordinate multiplication rule.

\begin{theorem}
\label{thm:galois-multiplication}
Let $B_1$ and $B_2$ be two ordered $\mathbb{F}_2$-bases of
$F=\mathbb{F}_{2^n}$. Then
\[
m_{B_1}=m_{B_2}
\]
if and only if $B_1$ and $B_2$ belong to the same Galois orbit.
\end{theorem}

\begin{proof}
See Appendix~\ref{app:proofs}.
\end{proof}

It follows that the number of distinct coordinate multiplication maps is
\[
\frac{|\operatorname{GL}_n(\mathbb{F}_2)|}{n}
=
\frac{1}{n}
\prod_{i=0}^{n-1}(2^n-2^i).
\]
In the case of $F=\mathbb{F}_8$, the $168$ ordered bases form
\[
\frac{168}{3}=56
\]
Galois orbits, corresponding to exactly $56$ distinct coordinate
multiplication maps. This orbit structure is the basis of the
train and test construction introduced in Section~\ref{sec:problem-setting}.

\section{Problem Setting}
\label{sec:problem-setting}

We now specialize the mathematical framework of
Section~\ref{sec:mathematical-framework} to
$F=\mathbb{F}_8$. As established above, the $168$ ordered bases of $F$ form $56$ Galois orbits of size three,
and the three bases within each orbit induce the same coordinate
multiplication map. We use this structure to study whether neural models can recognize and
use the relationship among bases in the same Galois orbit, and how
their behavior depends on the information and learning objective
provided during training. An extension to $50$
sampled Galois orbits in $\mathbb{F}_{16}$ is reported in
Appendix~\ref{app:f16}.

\subsection{Basis Representation}

For the experiments, we represent
\[
F=\mathbb F_8
\]
using the irreducible polynomial
\[
X^3+X+1
\]
and choose
\[
E=(1,\alpha,\alpha^2)
\]
as the fixed reference basis. For every ordered basis \(B\), we use
the matrix \(P_B\) defined in
Section~\ref{sec:mathematical-framework} as its representation. The
columns of \(P_B\) contain the coordinates of the elements of \(B\)
relative to \(E\), and
\[
[x]_E=P_B[x]_B
\]
for every \(x\in F\). Note that \(P_B\) completely describes \(B\)
relative to the fixed reference basis and is the representation of the
basis used throughout the experiments.

\subsection{Train/test split}
\label{sec:train-test-split}

For each experimental seed and each Galois orbit
\[
\mathcal O_j
=
\{B_j,\sigma(B_j),\sigma^2(B_j)\},
\qquad j=1,\ldots,56,
\]
we randomly select two of the three bases for training and reserve the
remaining basis for testing. As a result, there are $112$ training bases and
$56$ held-out test bases. The random selection is determined by the seed, and all conditions
evaluated under the same seed use the same split. Consequently, differences across seeds reflect both random model initialization and variation in which basis from each orbit is reserved for testing.

The training and test sets contain no common ordered bases. However,
for every test basis, the other two bases in the same Galois orbit are
included in the training set. Since all bases in one orbit induce the
same coordinate multiplication map, all \(56\) multiplication maps
evaluated at test time are already represented in the training set.
The test set contains no unseen multiplication rules. Only the basis
used to represent each rule changes at test time.

For every basis $B$, we enumerate all $8^2=64$ ordered pairs
$(x,y)\in F^2$ and define
\[
x_B=[x]_B,\qquad
y_B=[y]_B,\qquad
z_B=[xy]_B=m_B(x_B,y_B).
\]
By doing so, each training basis contributes its complete coordinate
multiplication table. The resulting datasets for multiplication contain
\[
112\times64=7168
\]
training examples and
\[
56\times64=3584
\]
test examples.

The same \(112/56\) basis split is used throughout the
experiments. Additional separation rules required by individual tasks
are described in the corresponding subsections.

\subsection{Basis conditioned multiplication}
\label{sec:basis conditioned-multiplication}

Our primary task is multiplication in the finite field
\(\mathbb F_8\) using different forms of information about the basis. For each example, the operands are
\[
x_B=[x]_B,\qquad y_B=[y]_B,
\]
and the prediction target is
\[
z_B=[xy]_B.
\]
We consider the six input conditions summarized in
Table~\ref{tab:multiplication-conditions}.

\begin{table*}[t]
\centering
\caption{
Input conditions for basis conditioned multiplication.
}
\label{tab:multiplication-conditions}
\begin{tabular}{llcl}
\toprule
\textbf{Condition}
& \textbf{Input}
& \textbf{Tokens} \\
\midrule

No orbit label, no \(P_B\)
& $(x_B,y_B)$
& $2$
\\

Orbit label, no \(P_B\)
& $(\ell_B,x_B,y_B)$
& $3$
 \\

\(P_B\), no orbit label
& $(P_B,x_B,y_B)$
& $11$
\\

Orbit label $+$ \(P_B\)
& $(\ell_B,P_B,x_B,y_B)$
& $12$
\\

Orbit label $+$ constant tokens
& $(\ell_B,\mathbf{0}_9,x_B,y_B)$
& $12$
 \\

Orbit label $+$ shuffled matrix
& $(\ell_B,\widetilde P_B,x_B,y_B)$
& $12$
 \\

\bottomrule
\end{tabular}
\end{table*}

\paragraph{Primary information conditions.}
The condition with operands alone receives no information about the
basis or its Galois orbit and therefore cannot determine which
multiplication map generated an example. 

The condition receiving only the Galois orbit label is given
\[
\ell_B\in\{1,\ldots,56\},
\]
which is shared by the three bases in the orbit of \(B\). Since each
Galois orbit corresponds to one coordinate multiplication map,
\(\ell_B\) directly identifies the multiplication rule associated
with the example. 

On the other hand,
in the condition without an orbit label, the model receives
\[
(P_B,x_B,y_B).
\]
The matrix \(P_B\) uniquely identifies the ordered basis, but it does
not state explicitly which Galois orbit contains that basis. This
condition tests whether the model can recover the relevant
multiplication rule from the basis matrix.

The condition receiving both inputs is given
the orbit label $\ell_B$ along with the basis matrix $P_B$. Comparing it with the condition using the orbit label alone measures
how the additional basis matrix changes performance after the
multiplication rule has already been identified.

\paragraph{Constant token control.}
To control for the nine additional token positions introduced by
\(P_B\), the constant token condition replaces its entries with
\[
\mathbf 0_9=(0,\ldots,0)\in\mathbb F_2^9,
\]
which is identical for every basis and every example. This condition has the same sequence length, token type, and
token positions as the condition receiving \(P_B\), but its nine
additional tokens contain no information that varies across bases.

\paragraph{Shuffled matrix control.}
In the shuffled matrix control, each basis \(B\) is assigned a fixed
matrix
\[
\widetilde P_B\in\mathrm{GL}_3(\mathbb F_2)
\]
by randomly permuting the \(168\) basis matrices. For each experimental
seed, the permutation is chosen so that
\[
\widetilde P_B\neq P_B
\]
for every basis \(B\). Every matrix is used exactly once, and the
assignment remains fixed across all \(64\) examples associated with a
basis. The shuffled condition preserves the full set of \(168\) basis
matrices, but assigns each matrix to a basis other than its original
one.

\paragraph{Purpose of the controls.}
The controls test whether performance is affected by additional token positions, variation across bases, or the correct relationship between
each basis and its matrix.
The constant token condition examines whether the additional token
positions alone can produce similar behavior. The shuffled matrix
condition examines whether the correct relationship between a basis
and its matrix provides an advantage when the same collection of
matrices is assigned to different bases. These
comparisons constrain the possible explanations for the observed behavior, but they do not by themselves isolate a unique causal mechanism.


\subsection{Galois orbit identification}
\label{sec:global-orbit-identification}
The conditions using basis matrices test indirectly whether a model
can extract information about the Galois orbit from \(P_B\) and use it
for multiplication. To examine this question independently of the
multiplication task, we remove the operands and product target and
formulate Galois orbit identification as a classification problem:
\[
\begin{aligned}
\text{Input:}\quad
& P_B,\\
\text{Target:}\quad
& \ell_B\in\{1,\ldots,56\}.
\end{aligned}
\]
Each Galois orbit is assigned an arbitrary categorical label
$\ell_B$. The label is invariant under the Galois action:
\[
\ell_B
=
\ell_{\sigma(B)}
=
\ell_{\sigma^2(B)}.
\]

For each experimental seed, the model is trained on the two
training bases from every orbit and evaluated on the remaining held-out
basis. It receives $112$ labeled basis matrices for training
and $56$ basis matrices for testing. To succeed, the
model must assign each test basis the orbit label shared by the two
training bases from the same orbit.

Each Galois orbit is assigned an arbitrary class label. These labels do
not encode how the corresponding bases are related. This task tests whether
the model can transfer these orbit labels to the test bases. In the
following subsection, we separately examine whether the model can
recognize the relation between two bases from the same orbit.

\subsection{Pairwise orbit recognition}
\label{sec:pairwise-orbit-recognition}
The Galois orbit identification task requires the model to assign one of \(56\) Galois orbit labels to each basis. We also consider a relational formulation that asks only whether two bases
belong to the same Galois orbit:
\[
\text{Input: }(P_B,P_{B'}),\qquad
\text{Target: }
\mathbf 1\{\mathcal O_B=\mathcal O_{B'}\}.
\]

\paragraph{Training pairs.}
Let $B_j^{(1)}$ and $B_j^{(2)}$ denote the two training bases from orbit
$\mathcal O_j$. For each orbit, we construct the two ordered positive
pairs
\[
\left(P_{B_j^{(1)}},P_{B_j^{(2)}}\right)
\quad\text{and}\quad
\left(P_{B_j^{(2)}},P_{B_j^{(1)}}\right).
\]
These are the only positive training pairs, and neither pair is
duplicated through oversampling. We additionally construct two distinct
negative pairs per orbit by pairing a training basis from
$\mathcal O_j$ with a training basis from a different orbit. Through all $56$ orbits, the training set contains $112$ positive and
$112$ negative pairs.

\paragraph{Held-out test pairs.}
Let $B_j^{(\mathrm{test})}$ denote the held-out basis from orbit
$\mathcal O_j$. For each orbit, the positive test pairs are
\[
\left(P_{B_j^{(\mathrm{test})}},P_{B_j^{(1)}}\right)
\quad\text{and}\quad
\left(P_{B_j^{(\mathrm{test})}},P_{B_j^{(2)}}\right).
\]
We also construct two negative test pairs by pairing
$P_{B_j^{(\mathrm{test})}}$ with matrices of training bases sampled from
other orbits. The test set contains $112$ positive and $112$
negative pairs. Every test pair contains a held-out basis matrix that
does not occur in the training pairs.

Table~\ref{tab:pairwise-datasets} summarizes the construction and class
balance of the training and test pairs.

\begin{table}[t]
\centering
\caption{
Balanced pairwise orbit recognition datasets. Counts are aggregated
across the $56$ Galois orbits.
}
\label{tab:pairwise-datasets}
\begin{tabular}{lccc}
\toprule
\textbf{Split}
& \textbf{Positive}
& \textbf{Negative}
& \textbf{Total} \\
\midrule
Training
& $112$
& $112$
& $224$ \\
Test
& $112$
& $112$
& $224$ \\
\bottomrule
\end{tabular}
\end{table}

\paragraph{Evaluation.}
The training and test sets contain equal numbers of positive and
negative pairs, so random guessing gives an expected accuracy of
\(1/2\). We report overall test accuracy together with separate
accuracies for positive and negative pairs. So the results show whether errors
occur more often for one type of pair.

The task evaluates whether the model can recognize that a held-out
basis and a training basis belong to the same Galois orbit. It does not
require the model to predict a Galois orbit label or construct a
canonical representative.

\subsection{Learning the Galois action}
\label{sec:learning-galois-action}
The preceding tasks attempt to recover orbit information either through
a Galois orbit identifier or through pairwise orbit recognition. We
next take a different approach. Instead of predicting an invariant
directly, we train the model to predict the transformation between
equivalent basis representations. Specifically, given $P_B$, the model
predicts $P_{\sigma(B)}$, where $\sigma(x)=x^2$ is the Frobenius
automorphism:
\[
\begin{aligned}
\text{Input:}\quad
& P_B,\\
\text{Target:}\quad
& P_{\sigma(B)}.
\end{aligned}
\]

\paragraph{Training data.}
For each experimental seed, the model is trained to predict
\[
P_B\longmapsto P_{\sigma(B)}
\]
for each of the \(112\) training bases. The matrices of the \(56\)
held-out bases are not used as model inputs during training and are
introduced as inputs only during evaluation. A held-out basis matrix may still appear as the prediction target for
a training basis from the same orbit. The evaluation tests whether
the learned transformation can be applied to a new input matrix, not
whether the model can predict a matrix that has never appeared anywhere
in the training data. Because \(P_{\sigma(B)}\) is a \(3\times3\)
binary matrix, the model predicts its nine entries using a binary loss
for each entry.

\paragraph{One step evaluation on held-out bases.}
Let $\widehat{\sigma}$ denote the learned transformation. For each
held-out basis $B^{(\mathrm{test})}$, we compare
\[
\widehat{\sigma}
\left(P_{B^{(\mathrm{test})}}\right)
\]
with the exact target
\[
P_{\sigma(B^{(\mathrm{test})})}.
\]
We report both entrywise bit accuracy and exact-matrix accuracy. An
exact prediction is counted only when all nine matrix entries are
correct.

\paragraph{Composition and cycle closure.}
The Frobenius automorphism of $\mathbb F_8$ has order three, so applying
it three times returns to the original basis:
\[
\sigma^3(B)=B.
\]
Starting from each held-out predictor input, we apply the
learned transformation repeatedly:
\[
P_B
\longmapsto
\widehat{\sigma}(P_B)
\longmapsto
\widehat{\sigma}^{\,2}(P_B)
\longmapsto
\widehat{\sigma}^{\,3}(P_B).
\]
After the first step, the predicted matrix
\[
\widehat{\sigma}(P_B)
\]
is used as the input to the next step. It is not replaced by the exact
matrix \(P_{\sigma(B)}\). We evaluate two-step composition by testing whether
\[
\widehat{\sigma}^{\,2}(P_B)
=
P_{\sigma^2(B)},
\]
and evaluate three-step cycle closure by testing whether
\[
\widehat{\sigma}^{\,3}(P_B)
=
P_B.
\]

The first evaluation tests a single application of the learned
transformation. The remaining evaluations test whether repeated
applications reach the expected members of the Galois orbit and
return to the original basis after three applications.

\subsection{Canonicalization and downstream multiplication}
\label{sec:canonicalization}

The learned transformation \(\widehat{\sigma}\) is trained to map a
basis matrix \(P_B\) to the matrix \(P_{\sigma(B)}\) of the
Frobenius transformed basis. We use this transformation
to construct a single canonical representative for the orbit and test
whether it can support multiplication on held-out bases.

\paragraph{Ground truth and learned canonicalization.}
For a basis $B$, the exact Frobenius action generates
\[
P_B,\qquad
P_{\sigma(B)},\qquad
P_{\sigma^2(B)}.
\]
To compare the matrices lexicographically, we read the entries of each
matrix row by row and define the ground truth canonical representative
as
\[
C(B)
=
\min_{\mathrm{lex}}
\left\{
P_B,
P_{\sigma(B)},
P_{\sigma^2(B)}
\right\}.
\]
Because $\sigma$ only permutes the three matrices in this set, we have
\[
C(B)
=
C(\sigma(B))
=
C(\sigma^2(B)).
\]
Consequently, all bases in the same orbit have the same ground truth
canonical representative. Each distinct ground truth canonical representative is assigned a
unique identifier. We denote the identifier associated with \(C(B)\)
by $$c_B.$$

Note that we also have 
\[
c_B
=
c_{\sigma(B)}
=
c_{\sigma^2(B)}
\]
since \(C(B)\) is constant on each Galois orbit. 

We refer to \(c_B\) as the ground truth canonical identifier. This
construction provides the ground truth reference for the learned
canonicalization procedure below.

For learned canonicalization, we replace the exact Frobenius action
with the learned transformation $\widehat{\sigma}$ and define
\[
\widehat C(B)
=
\min_{\mathrm{lex}}
\left\{
P_B,
\widehat{\sigma}(P_B),
\widehat{\sigma}^{\,2}(P_B)
\right\}.
\]
Here,
\[
\widehat{\sigma}^{\,2}(P_B)
=
\widehat{\sigma}
\left(
\widehat{\sigma}(P_B)
\right),
\]
so the second predicted orbit member is obtained by feeding the first
prediction back into the same learned transformation. Unlike $C(B)$, the learned canonical representative
$\widehat C(B)$ is not guaranteed to be identical for all bases in the
same orbit.

\paragraph{Learned canonical identifiers and exact lookup.}

For each of the \(112\) training bases \(B\), the model produces
\[
\widehat{\sigma}(P_B)
\qquad\text{and}\qquad
\widehat{\sigma}^{\,2}(P_B).
\]
Outside the model, we apply the fixed lexicographic rule to
\[
P_B,\qquad
\widehat{\sigma}(P_B),\qquad
\widehat{\sigma}^{\,2}(P_B)
\]
to obtain the learned canonical representative \(\widehat C(B)\).

We assign a unique identifier to each distinct learned canonical
representative obtained from the training bases and store these
representatives and their identifiers in a lookup table. For a
training basis \(B\), we denote the identifier associated with
\(\widehat C(B)\) by
\[
\widehat c_B.
\]
Thus, for any two training bases \(B\) and \(B'\),
\[
\widehat c_B=\widehat c_{B'}
\quad\Longleftrightarrow\quad
\widehat C(B)=\widehat C(B').
\]
We refer to \(\widehat c_B\) as the learned canonical identifier.

For a held-out basis \(B\), we obtain \(\widehat C(B)\) using the same
procedure and compare it with the distinct learned canonical
representatives obtained from the training bases. A match requires
equality of all nine binary entries.

If an exact match is found, we assign the held-out basis the learned
canonical identifier associated with the matching training
representative. If no exact match is found, we count canonicalization
as a failure. No approximate matching is used in the primary
evaluation.

\paragraph{Downstream multiplication.}
The multiplication model is trained on the complete multiplication
tables of the $112$ training bases. Its input and target are
\[
\begin{aligned}
\text{Input:}\quad
& (\widehat c_B,x_B,y_B),\\
\text{Target:}\quad
& z_B=[xy]_B.
\end{aligned}
\]
At test time, we first use exact lookup to obtain a learned canonical
identifier for each held-out basis. If a match is found, the
corresponding identifier is provided to the multiplication model with \(x_B\) and \(y_B\). If no matching learned canonical
representative is found among the training bases, lookup fails and all
\(64\) multiplication examples from that basis are counted as incorrect.

The reported held-out accuracy includes errors from both stages.
Incorrect outputs from the multiplication model are counted as errors,
and all \(64\) examples from a basis are counted as incorrect when
exact lookup fails.

\paragraph{Exact evaluation and approximate recovery.}
For each orbit \(\mathcal O_j\), let \(B_j^{(1)}\) denote a fixed one
of the two training bases and let \(B_j^{(\mathrm{test})}\) denote the
held-out basis. The exact canonical match rate is the fraction of the
\(56\) orbits for which
\[
\widehat C\left(B_j^{(\mathrm{test})}\right)
=
\widehat C\left(B_j^{(1)}\right).
\]
The choice of \(B_j^{(1)}\) is made before evaluation and does not
depend on the resulting match.

This metric tests whether learned canonicalization produces exactly
the same representative for a held-out basis and a fixed training
basis from the same orbit. For downstream multiplication, the lookup instead compares the
held-out representative with the representatives obtained from all
training bases. We report training exact accuracy for the downstream
multiplication model and held-out exact accuracy for the complete
downstream procedure.

In addition to the primary exact lookup procedure, we evaluate a
nearest Hamming recovery rule as a secondary analysis. If exact lookup fails, we compare the
held-out representative with all representatives obtained from the
training bases and select the one that differs in the fewest binary
entries. The held-out basis is then assigned the identifier associated
with that representative. This recovery rule is evaluated separately. It is not used when
computing the primary exact canonical match rate or the primary
downstream multiplication accuracy.

\subsection{Decomposed multiplication pipeline}
\label{sec:decomposed-pipeline}
As a structured diagnostic, we use the decomposition of multiplication
given in Equation~\ref{eq:basis-decomposition}:
\[
m_B(u,v)
=
P_B^{-1}m_E(P_Bu,P_Bv).
\]
The decomposition first converts the two operands from the coordinates
of \(B\) to those of the reference basis \(E\), then performs
multiplication in \(E\), and finally converts the product back to the
coordinates of \(B\). This sequence of operations is specified
explicitly. It is not discovered by the model during training.

\paragraph{Learned components.}
We train two separate neural networks. The first network, which we call
the \emph{matrix vector module}, learns binary matrix vector
multiplication:
\[
\begin{aligned}
\text{Input:}\quad
& (M,u),\\
\text{Target:}\quad
& Mu,
\end{aligned}
\]
where $M\in\mathbb F_2^{3\times3}$,
$u\in\mathbb F_2^3$, and the multiplication is performed over
$\mathbb F_2$.

The second network, which we call the
\emph{reference basis multiplication module}, learns multiplication in $\mathbb{F}_8$ in the fixed basis $E$:
\[
\begin{aligned}
\text{Input:}\quad
& (a,b),\\
\text{Target:}\quad
& m_E(a,b),
\end{aligned}
\]
where $a,b\in\mathbb F_2^3$ are coordinate vectors in the reference
basis.

The matrix vector module is trained on all eight binary vectors for
each retained training matrix. The reference basis multiplication
module is trained on all $8^2=64$ ordered pairs in $\mathbb F_8$. For every basis matrix $P_B$, the inverse $P_B^{-1}$ is computed exactly
by a deterministic algebraic procedure over $\mathbb F_2$, so the matrix
inversion is not learned. The resulting matrices $P_B$ and
$P_B^{-1}$ are provided as inputs to the learned matrix vector module,
which predicts their action on coordinate vectors.

\paragraph{Matrix training with disjoint transformations.}
At test time, multiplication in a held-out basis \(B\) requires the
matrices
\[
P_B
\qquad\text{and}\qquad
P_B^{-1}.
\]
To prevent either matrix from appearing during training, we first
collect the basis matrices and their inverses from the \(112\) training
bases and remove duplicates. Next, we remove any matrix that is equal
to \(P_B\) or \(P_B^{-1}\) for one of the \(56\) held-out bases.
Consequently, neither matrix required for a held-out basis occurs in
the training data of the matrix vector module.

For every matrix that remains, the module is trained on all eight
vectors in \(\mathbb F_2^3\). Therefore, the vectors used during evaluation
have appeared during training, but the matrices required by the
held-out bases have not. The module for multiplication in the reference basis is
trained separately on the complete multiplication table associated
with \(E\).

\paragraph{Composition on held-out bases.}
For each held-out basis $B$, the matrix vector module first applies
$P_B$ to both input vectors:
\[
u_E=P_Bu,
\qquad
v_E=P_Bv.
\]
The reference basis multiplication module then predicts
\[
z_E=m_E(u_E,v_E),
\]
and the matrix vector module applies $P_B^{-1}$ to transform the
predicted product back to basis $B$:
\[
z_B=P_B^{-1}z_E.
\]
Both $P_B$ and $P_B^{-1}$ are absent from the training data of the
matrix vector module.

We evaluate the composed pipeline on all \(64\) multiplication examples
for each of the \(56\) held-out bases. We report bit accuracy and exact
output accuracy. An output is counted as exact only when all three
predicted bits are correct.

\paragraph{Interpretation.}
Unlike the preceding tasks, this experiment provides the algebraic
decomposition in advance and trains separate models for its two
operations. It tests whether these learned operations can be combined
to perform multiplication on held-out bases when the required
transformation matrices are excluded from training. However, it does not test
whether a model can discover the decomposition on its own.

\section{Experimental Setup}
\label{sec:experimental-setup}
We now describe the experimental implementation of the learning tasks
introduced in Section~\ref{sec:problem-setting}. We first specify how
the algebraic objects are represented and tokenized as model inputs,
and then describe the model architectures, training procedures, and
evaluation metrics used across the experiments. Unless otherwise stated, all experiments are conducted over
\(\mathbb F_8\) using the \(112/56\) split of bases defined in
Section~\ref{sec:train-test-split}. Additional training and evaluation rules for pair construction,
learned transformations, and decomposed computation are described in
the relevant subsections.

\subsection{Input representation and tokenization}
\label{sec:tokenization}
All model inputs are represented as token sequences. For the
multiplication tasks, each coordinate vector
\[
x_B,y_B\in\mathbb F_2^3
\]
is encoded as a single token with one of eight possible values. Thus,
each operand occupies one token position rather than three separate
binary token positions. The target
\[
z_B=[xy]_B
\]
is represented as an output with eight possible classes.

The Galois orbit label
\[
\ell_B\in\{1,\ldots,56\}
\]
is represented by a single categorical token. Because the label is
shared by the three bases in an orbit, the same token is used for all
bases that induce the same coordinate multiplication map.

For downstream multiplication, each learned canonical identifier
\(\widehat c_B\) is represented as a single token.

\paragraph{Basis matrix tokens and controls.}
The basis matrix
\[
P_B\in\mathrm{GL}_3(\mathbb F_2)
\]
is flattened in a fixed row major order:
\[
P_B
\longmapsto
(p_1,\ldots,p_9),
\qquad
p_i\in\mathbb F_2.
\]
Each matrix entry occupies one binary token position. Thus, $P_B$ is
represented as nine binary tokens rather than as a single categorical
matrix token or a continuous feature vector.

The constant token control uses nine zeros in the same positions, while
the shuffled matrix control uses the nine entries of
\(\widetilde P_B\). The complete input sequences for all conditions are
summarized in Table~\ref{tab:multiplication-conditions}.

\paragraph{Galois action prediction.}
For learning the Frobenius action, the nine binary entries of $P_B$
form the input sequence. The prediction target consists of the nine
binary entries of $P_{\sigma(B)}$. Both the input and output hold the entrywise binary structure of the basis matrices.

\paragraph{Decomposed components.}
For the learned matrix vector module, the nine entries of
$M\in\mathbb F_2^{3\times3}$ and the three entries of
$u\in\mathbb F_2^3$ form a sequence of $12$ binary tokens. The model predicts the three binary entries of $Mu$.

For the learned reference basis multiplication module, each operand is
a vector in \(\mathbb F_2^3\), and its three entries are represented
as separate binary tokens. The two operands hence form a sequence
of six tokens. The model predicts the three binary entries of their product in the fixed
reference basis $E$.

\subsection{Model architecture and training}
\label{sec:model-training}
All learning tasks use the same Transformer backbone with two layers
of self attention, model dimension
\[
d_{\mathrm{model}}=32,
\]
and feedforward dimension
\[
d_{\mathrm{FF}}=64.
\]
Each layer applies layer normalization before single head scaled dot
product self attention and before a feedforward network that operates
independently at each token position. A residual connection surrounds
each of the two components.

After the Transformer layers, the token representations are flattened
and passed through a readout network with one hidden layer, followed by
an output head chosen for each task. The Transformer dimensions remain
fixed across tasks, while the readout width, output head, training
budget, and batch size vary by task. Table~\ref{tab:model-config}
summarizes the model and training configurations.

\begin{table*}[t]
\centering
\caption{
Model and training configurations used in the main experiments.
All models use an initial learning rate of $3\times10^{-3}$.
}
\label{tab:model-config}
\small
\begin{tabular}{lccccccc}
\toprule
\textbf{Task/condition}
& $\mathbf{L}$
& $\mathbf{d_{\mathrm{model}}}$
& $\mathbf{d_{\mathrm{FF}}}$
& \textbf{Hidden}
& \textbf{Epochs}
& \textbf{Batch}
& \textbf{Seeds} \\
\midrule

\multicolumn{8}{l}{\textit{Basis conditioned multiplication}} \\

\quad Orbit label, no $P_B$
& 2 & 32 & 64 & 128 & 150 & 128 & 5 \\

\quad No orbit label, no $P_B$
& 2 & 32 & 64 & 128 & 50 & 128 & 5 \\

\quad $P_B$, no orbit label
& 2 & 32 & 64 & 128 & 800 & 128 & 8 \\

\quad Orbit label $+\,P_B$
& 2 & 32 & 64 & 128 & 400 & 128 & 8 \\

\quad Orbit label $+$ constant tokens
& 2 & 32 & 64 & 128 & 400 & 128 & 8 \\

\quad Orbit label $+$ shuffled matrix
& 2 & 32 & 64 & 128 & 400 & 128 & 8 \\

\midrule

Galois orbit identification
& 2 & 32 & 64 & 64 & 300 & 32 & 5 \\

Pairwise orbit recognition
& 2 & 32 & 64 & 64 & 200 & 64 & 5 \\

Galois action prediction
& 2 & 32 & 64 & 64 & 150 & 32 & 5 \\

Downstream multiplication
& 2 & 32 & 64 & 64 & 150 & 64 & 5 \\

Matrix-vector module
& 2 & 32 & 64 & 64 & 60 & 64 & 5 \\

Reference-basis multiplication module
& 2 & 32 & 64 & 64 & 200 & 64 & 5 \\

\bottomrule
\end{tabular}
\end{table*}

\paragraph{Parameter counts across input conditions.}
For basis conditioned multiplication, the Transformer backbone and
hidden readout width are held fixed across the six primary input
conditions. The token representations are flattened before entering the readout
layer, so the readout input dimension depends on the sequence length.
Conditions with different sequence lengths thus have different
total parameter counts. The three conditions that combine the Galois
orbit label with \(P_B\), constant tokens, or \(\widetilde P_B\) all
use sequences of length \(12\) and have identical architectures and
parameter counts.

\paragraph{Output heads and losses.}
Basis conditioned multiplication and downstream multiplication use
eight output classes with softmax cross entropy loss. Galois orbit
identification uses \(56\) output classes, while pairwise orbit
recognition uses a single binary logit.

Prediction of the Galois action uses nine binary logits corresponding
to the entries of \(P_{\sigma(B)}\). The matrix vector module and the
reference basis multiplication module each use three binary logits.
All tasks that predict binary entries are trained using binary cross
entropy applied to the logits.

\paragraph{Optimization.}
All models are trained with Adam using an initial learning rate of
\(3\times10^{-3}\). For basis conditioned multiplication, Galois orbit
identification, pairwise orbit recognition, and downstream
multiplication using learned canonicalization, the learning rate is
multiplied by \(0.3\) after one half and three quarters of the training
epochs. For prediction of the Galois action and the two learned
components of the decomposed pipeline, the learning rate is multiplied
by \(0.3\) only after one half of the training epochs.

All models are trained for the complete number of epochs listed in
Table~\ref{tab:model-config} without early stopping. We report the
metrics from the final epoch for every seed.

\subsection{Evaluation metrics}
\label{sec:evaluation}
For each experiment, we record the final metrics and training history
for every seed. Final results are summarized using the mean and population standard
deviation across seeds. The learning curve figures show a separate
trajectory for each seed without averaging across runs.

\paragraph{Multiplication accuracy.}
Exact accuracy is the primary metric for basis conditioned
multiplication and downstream multiplication. A prediction is counted
as correct only when the predicted field element agrees exactly with
the target.

For the decomposed multiplication pipeline, the output is represented
by three binary entries. We report both entrywise bit accuracy
and exact output accuracy. An output is counted as exact only when all
three predicted bits are correct.

For downstream multiplication using learned canonicalization, a
held-out basis whose learned canonical representative has no exact
match among the training representatives contributes \(64\) incorrect
predictions. The reported test accuracy includes both
multiplication errors and failures of exact lookup.

\paragraph{Baseline using the operands alone.}
When neither the Galois orbit label nor the basis matrix is provided,
the model receives only \((x_B,y_B)\). We compute a modal baseline by
predicting the most frequent target for each operand pair across the
training bases. This calculation gives an exact accuracy of \(0.3438\)
for the \(\mathbb F_8\) datasets, well above \(1/8\) for uniform guessing.

\paragraph{Metrics for orbit recognition.}
Galois orbit identification is evaluated using classification accuracy
over the \(56\) Galois orbit identifiers.

For pairwise orbit recognition, we report overall test accuracy
together with separate accuracies for positive and negative pairs.
Positive pairs contain two bases from the same orbit, while negative
pairs contain bases from different orbits. Because the test set
contains equal numbers of positive and negative pairs, random guessing
has an expected accuracy of \(1/2\). The separate accuracies show
whether the model performs differently on the two types of pairs.

\paragraph{Metrics for prediction of the Galois action.}
The learned transformation $\widehat{\sigma}$ is evaluated on the $56$
held-out predictor inputs. 
For each held-out basis \(B\), we apply the learned transformation
repeatedly and compare the result after \(k\) steps with
\[
P_{\sigma^k(B)},
\qquad
k=1,2,3.
\]
The exact accuracy at step \(k\) is the fraction of held-out bases for
which
\[
\widehat{\sigma}^{\,k}(P_B)
=
P_{\sigma^k(B)}.
\]
The cases \(k=1\), \(k=2\), and \(k=3\) correspond to prediction after
one application, composition after two applications, and cycle closure
after three applications. Since \(\sigma^3(B)=B\), the cycle closure
condition is
\[
\widehat{\sigma}^{\,3}(P_B)=P_B.
\]

For \(k=1\), we additionally report bit accuracy, defined as the
fraction of the nine predicted matrix entries that agree with the
corresponding entries of \(P_{\sigma(B)}\). 

Accuracy is evaluated
separately at each step, so correctness at step \(k\) does not require
the preceding predictions to have been correct.

For the two-step and three-step metrics, each prediction is fed back
into the same learned transformation. The exact intermediate matrices
are not supplied.

\paragraph{Metrics for canonicalization and downstream multiplication.}
The primary canonicalization metric is the exact canonical match rate
defined in Section~\ref{sec:canonicalization}. A held-out basis is
counted as a match only when all nine entries of its learned canonical
representative agree with those of the representative obtained from
the fixed training basis in the same orbit.

We also report downstream multiplication accuracy using learned
canonical identifiers obtained through exact lookup. Approximate
recovery based on Hamming distance is reported separately and is not
used in the primary canonical match rate or the primary downstream
multiplication accuracy.

\section{Experimental Results and Analysis}
\label{sec:results}
We now report model performance on held-out basis representations using
the experimental protocols defined in
Sections~\ref{sec:problem-setting} and
\ref{sec:experimental-setup}.  We first examine how different forms of
basis information affect multiplication and then compare Galois orbit
identification with pairwise orbit recognition. We next evaluate the
decomposed multiplication pipeline under strict separation between the
matrices used for training and those required during evaluation.
Finally, we test whether the learned Galois action generalizes to
held-out predictor inputs and whether it supports exact
canonicalization and downstream multiplication.

\subsection{Results for basis conditioned multiplication}
\label{sec:results-basis-conditioning}

Table~\ref{tab:basis-conditioning-results} reports the final exact
accuracy for the six conditions using different forms of information
about the basis.

\begin{table}[t]
\centering
\caption{
Final exact accuracy for basis conditioned multiplication. Values are
the mean and population standard deviation across experimental seeds.
}
\label{tab:basis-conditioning-results}
\small
\begin{tabular}{lccc}
\toprule
Condition
& Seeds
& Training exact
& Held-out exact \\
\midrule
Orbit label, no \(P_B\)
& \(5\)
& \(1.0000 \pm 0.0000\)
& \(1.0000 \pm 0.0000\) \\

No orbit label, no \(P_B\)
& \(5\)
& \(0.3438 \pm 0.0000\)
& \(0.3438 \pm 0.0000\) \\

\(P_B\), no orbit label
& \(8\)
& \(0.3418 \pm 0.0052\)
& \(0.3419 \pm 0.0050\) \\

Orbit label \(+\,P_B\)
& \(8\)
& \(0.7295 \pm 0.3038\)
& \(0.6795 \pm 0.2659\) \\

Orbit label \(+\) constant tokens
& \(8\)
& \(0.5582 \pm 0.2850\)
& \(0.5582 \pm 0.2850\) \\

Orbit label \(+\) shuffled matrix
& \(8\)
& \(0.6130 \pm 0.2888\)
& \(0.5808 \pm 0.2559\) \\
\bottomrule
\end{tabular}
\end{table}

\paragraph{Galois orbit label and basis information.}
When the model receives the Galois orbit label with the two operands,
but without \(P_B\), it achieves perfect training and held-out
accuracy. This result follows from the fact that the three bases in
each orbit share the same coordinate multiplication table. Because the
complete tables of the two training bases are included, every input
\[
(\ell_B,x_B,y_B)
\]
used for a held-out basis also occurs during training with the same
target. Testing changes the basis representation but does not introduce
a new multiplication rule.

In contrast, when neither the orbit label nor \(P_B\) is provided, the
model receives only
\[
(x_B,y_B).
\]
The same operand pair can have different targets in different Galois orbits, but the operands alone do not identify the relevant multiplication rule. The trained model reaches an exact accuracy of
\(0.3438\) for all five seeds, equal to the accuracy obtained by
predicting the most frequent target for each operand pair.

Providing \(P_B\) without the orbit label produces nearly the same
result. Although \(P_B\) uniquely specifies the basis, both training
and held-out exact accuracy remain close to \(0.3438\). Under the
present architecture and training protocol, the model does not
reliably figure out the corresponding multiplication table from \(P_B\).
This means that the model does not successfully use the mathematical
information contained in the basis matrix under this condition.

\paragraph{Adding the basis matrix to the orbit label.}
When the model receives both the Galois orbit label and \(P_B\), it
reaches \(0.7295\pm0.3038\) training exact accuracy and
\(0.6795\pm0.2659\) held-out exact accuracy. This is significantly
lower and more variable than the perfect accuracy obtained with the
orbit label and operands alone.

The orbit label already identifies the multiplication table, so
\(P_B\) is mathematically redundant in this condition. However, the
difference in performance cannot be attributed to the information in
\(P_B\) alone. Adding its nine entries increases the sequence length
from \(3\) to \(12\). Because the token representations are flattened
before the readout layer, it also changes the readout input dimension
and the total parameter count. We next compare this condition
with two controls that have the same sequence length and architecture.

\paragraph{Constant and shuffled matrix controls.}
The constant token condition replaces the nine entries of \(P_B\) with
zeros. The shuffled matrix condition instead uses the entries of
\(\widetilde P_B\). It uses the same collection of matrices as the
\(P_B\) condition but breaks the correct relationship between each
basis and its matrix.

The constant token condition reaches \(0.5582\pm0.2850\) on both the
training and held-out sets. The shuffled matrix condition reaches
\(0.6130\pm0.2888\) in training and \(0.5808\pm0.2559\) on held-out
bases. The true \(P_B\) condition has the highest mean held-out
accuracy, exceeding the constant token condition by approximately
\(0.12\) and the shuffled matrix condition by approximately \(0.10\).
These differences are small relative to the variation across seeds,
and strong runs also occur in both control conditions. Hence, the results do not establish that the correct algebraic relationship
between a basis and \(P_B\) provides a solid advantage.

In the constant token condition, every held-out input and its target
already occur in the training data. This explains why training and held-out accuracy are identical within each seed. Variation across seeds reflects different optimization outcomes, but not a change in the
inputs encountered during evaluation. In the conditions using \(P_B\) or \(\widetilde P_B\), the matrix used
as input for a held-out basis differs from those used for the training
bases in the same orbit. Evaluation consequently includes input
sequences that do not occur during training. Training accuracy also
varies widely across seeds, so the instability cannot be attributed
solely to the new inputs introduced during evaluation.

\paragraph{Learning dynamics across seeds.}
Figure~\ref{fig:length-matched-test-curves} shows held-out exact
accuracy during training for the three conditions with equal sequence
length. All three conditions vary significantly across seeds. Some
runs remain near the floor of \(0.3438\), while others reach
intermediate or nearly perfect accuracy. Several curves remain at the
floor and overlap visually. In the condition using the true basis matrix \(P_B\), one run initially
reaches high accuracy but later returns to the floor. Thus, strong
performance during an intermediate stage of training does not
necessarily persist to the final epoch.

\begin{figure*}[t]
    \centering
    \includegraphics[width=\textwidth]
    {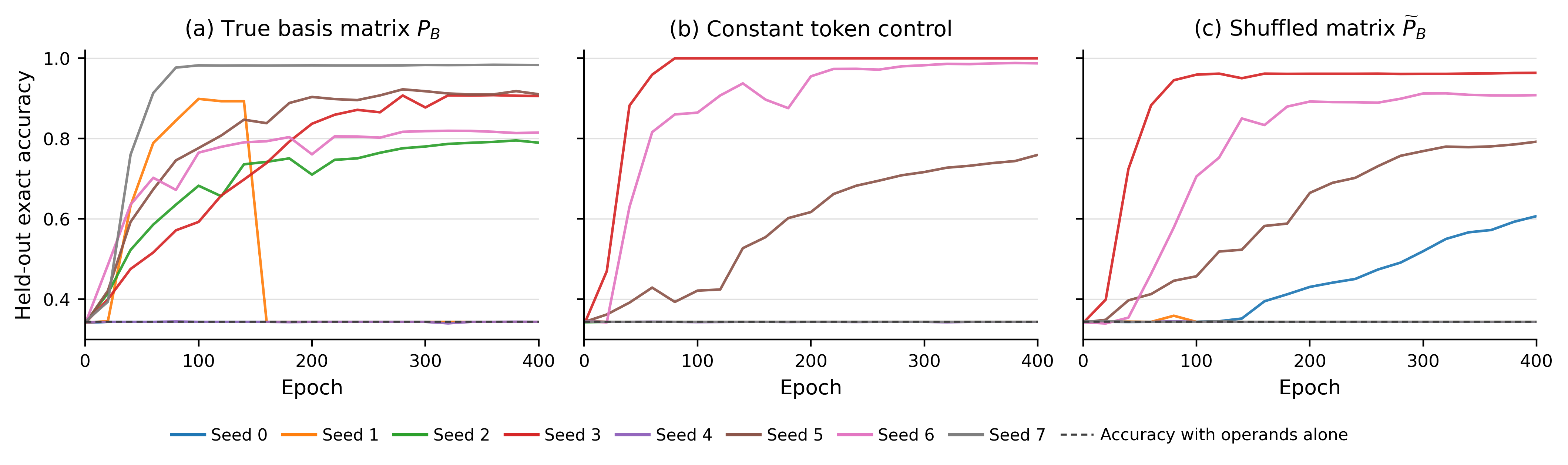}
    \caption{
    Held-out exact accuracy during training for the three conditions with
equal sequence length. The dotted
horizontal line marks the floor of \(0.3438\). Several curves remain
at the floor and overlap visually
    }
    \label{fig:length-matched-test-curves}
\end{figure*}

\paragraph{Summary.}
In summary, the results show that the Galois orbit label alone
supports perfect held-out performance, whereas adding nine token
positions produces highly variable outcomes. The true \(P_B\)
condition has the highest mean among the three conditions with equal
sequence length, but strong and weak runs also occur with constant
tokens and shuffled matrices. Thus, the results do not establish
that the correct algebraic relationship between \(P_B\) and its basis
provides a reliable advantage. The controls rule out some simple explanations, but they do not show
exactly why performance varies across seeds and conditions. Complete histories of loss, training accuracy, and held-out accuracy
for every seed are provided in
Appendix~\ref{app:f8-basis-curves}.

\subsection{Results for orbit recognition}
\label{sec:results-orbit-recognition}

We next evaluate whether models can recognize Galois orbit structure
in held-out basis representations. Table~\ref{tab:orbit-recognition-results}
compares Galois orbit identification with pairwise orbit recognition.
For the pairwise task, the table reports both overall accuracy and the
separate accuracies on positive and negative held-out pairs. 

\begin{table*}[t]
\centering
\caption{
Galois orbit identification and pairwise orbit recognition. Values are
means and population standard deviations across five experimental
seeds.
}
\label{tab:orbit-recognition-results}
\small
\begin{tabular}{lcccc}
\toprule
Task
& Overall train
& Overall held-out
& Positive held-out
& Negative held-out \\
\midrule
Galois orbit identification
& $1.0000 \pm 0.0000$
& $0.0179 \pm 0.0196$
& ---
& --- \\

Pairwise orbit recognition
& $1.0000 \pm 0.0000$
& $0.8688 \pm 0.0362$
& $0.9286 \pm 0.0196$
& $0.8089 \pm 0.0546$ \\
\bottomrule
\end{tabular}
\end{table*}

\paragraph{Galois orbit identification.}
The Galois orbit identification model reaches perfect training
accuracy but obtains only \(0.0179\pm0.0196\) accuracy on held-out
bases, approximately equal to the random guessing level of
\(1/56\approx0.0179\). The model fits the Galois orbit label of the training bases but
does not assign the same label to the held-out bases. This occurs
even though the orbit of every held-out basis contains two training
bases with the same label.

\paragraph{Pairwise orbit recognition.}
The pairwise model is trained and evaluated using balanced sets of
unique positive and negative ordered pairs. It reaches perfect overall
training accuracy and achieves \(0.8688\pm0.0362\) overall accuracy on
held-out pairs, considerably above the chance level of \(0.5\). On the held-out set, accuracy is \(0.9286\pm0.0196\) for positive pairs
and \(0.8089\pm0.0546\) for negative pairs.

This shows that the model recognizes most pairs containing a held-out basis and a
training basis from the same orbit. Its lower accuracy on negative
pairs indicates that more errors occur when distinguishing bases from
different orbits than when identifying bases from the same orbit.

To summarize, these results show that the failure of Galois orbit identification does not imply that orbit structure is inaccessible to the model. The model fails to assign the fixed Galois orbit identifiers
to held-out bases, but it can recognize whether a held-out basis and a
training basis belong to the same orbit. The ability to transfer orbit information depends strongly on the
learning objective.

\subsection{Results for the decomposed multiplication pipeline}
\label{sec:results-decomposed-pipeline}
We next report the results for the decomposed multiplication pipeline
defined in Section~\ref{sec:decomposed-pipeline}. For every held-out basis $B$,
neither $P_B$ nor $P_B^{-1}$ appears in the training data of the
matrix vector module. Table~\ref{tab:decomposed-pipeline-results} reports the resulting
held-out multiplication performance.

\begin{table}[t]
\centering
\caption{
Performance of the decomposed multiplication pipeline on \(56\)
held-out bases under strict separation. The final row reports the mean
and population standard deviation across five seeds.
}
\label{tab:decomposed-pipeline-results}
\small
\begin{tabular}{ccc}
\toprule
Seed
& Bit accuracy
& Exact accuracy \\
\midrule
$0$
& $0.9820$
& $0.9637$ \\

$1$
& $0.8582$
& $0.7221$ \\

$2$
& $0.9659$
& $0.9286$ \\

$3$
& $0.9740$
& $0.9436$ \\

$4$
& $1.0000$
& $1.0000$ \\
\midrule
Mean $\pm$ SD
& $0.9560 \pm 0.0502$
& $0.9116 \pm 0.0977$ \\
\bottomrule
\end{tabular}
\end{table}

\paragraph{Held-out multiplication.}
The decomposed pipeline reaches an average bit accuracy of
\(0.9560\pm0.0502\) and an average exact accuracy of
\(0.9116\pm0.0977\) on held-out bases. Performance varies across
seeds. Four of the five runs obtain exact accuracy above \(0.92\),
including one run with perfect exact accuracy. But Seed~1 only reaches
\(0.7221\). The lower exact accuracy relative to bit accuracy reflects
the stricter exact criterion, which requires all three predicted output
bits to be correct simultaneously.

\paragraph{Effect of the transform-disjoint protocol.}
For every held-out basis \(B\), the matrices \(P_B\) and \(P_B^{-1}\)
are new inputs to the matrix vector module because neither matrix is
used during its training. The strong held-out performance of the
complete pipeline does not depend on direct training with
these particular matrices. The result is consistent with the matrix
vector module learning an operation that transfers to matrices outside
its training set.

The matrix vector module is used three times for every multiplication
example. An error in any of these applications can affect the final
output, which provides one possible explanation for the observed
variation from seed to seed.

\subsection{Held-out evaluation of the learned Galois action}
\label{sec:results-galois-action}
We evaluate the learned transformation \(\widehat{\sigma}\) on the
\(56\) held-out basis matrices that are not used as predictor inputs
during training. Table~\ref{tab:galois-action-results} reports
prediction after one application, composition after two applications,
and cycle closure after three applications for each experimental seed.

\begin{table*}[t]
\centering
\caption{
Prediction, composition, and cycle closure accuracies of the learned
Galois action on the \(56\) held-out predictor inputs. The final row
reports the mean and population standard deviation across five
experimental seeds.
}
\label{tab:galois-action-results}
\small
\begin{tabular}{ccccc}
\toprule
Seed
& One-step bit
& One-step exact
& Two-step exact
& Three-step closure \\
\midrule
$0$
& $0.9980$
& $0.9821$
& $0.9821$
& $0.9821$ \\

$1$
& $1.0000$
& $1.0000$
& $1.0000$
& $1.0000$ \\

$2$
& $1.0000$
& $1.0000$
& $1.0000$
& $1.0000$ \\

$3$
& $1.0000$
& $1.0000$
& $1.0000$
& $1.0000$ \\

$4$
& $1.0000$
& $1.0000$
& $1.0000$
& $1.0000$ \\
\midrule
Mean $\pm$ SD
& $0.9996 \pm 0.0008$
& $0.9964 \pm 0.0071$
& $0.9964 \pm 0.0071$
& $0.9964 \pm 0.0071$ \\
\bottomrule
\end{tabular}
\end{table*}

\paragraph{Prediction after one application.}
On held-out predictor inputs, the learned transformation achieves
\(0.9996\pm0.0008\) bit accuracy and \(0.9964\pm0.0071\) exact matrix
accuracy. Four seeds predict the transformed matrix exactly for all
\(56\) held-out bases. For Seed~0, \(55\) of the \(56\) matrices are
predicted exactly, with only one incorrect matrix entry across the
complete held-out set.

The learned transformation remains accurate when held-out basis
matrices are used as predictor inputs. These matrices are excluded from
the training inputs, although some may appear as prediction targets
during training.

\paragraph{Composition and cycle closure.}
Applying the same learned transformation twice gives an exact accuracy
of \(0.9964\pm0.0071\). Applying it three times gives the same cycle
closure accuracy of \(0.9964\pm0.0071\). Four seeds satisfy both
diagnostics for every held-out basis, while Seed~0 succeeds for
\(55\) of the \(56\) held-out bases.

The equality of the results after one, two, and three applications
shows that repeated use of the transformation introduces no additional
failures beyond the single initial error in Seed~0. The learned
transformation reproduces the expected cycle of length three
almost perfectly on held-out predictor inputs. These results support
its use in learned canonicalization, which requires both
\(\widehat{\sigma}(P_B)\) and
\(\widehat{\sigma}^{\,2}(P_B)\).

\subsection{Results for canonicalization and downstream multiplication}
\label{sec:results-canonicalization}
We examine whether learned canonicalization maps training and held-out
bases from the same Galois orbit to identical representatives and
whether identifiers assigned through exact lookup support downstream
multiplication. Table~\ref{tab:canonicalization-results} presents the
primary exact matching results together with a secondary recovery
result based on Hamming distance.

\begin{table*}[t]
\centering
\caption{
Canonicalization and downstream multiplication results across five
experimental seeds. The final column reports the mean and population
standard deviation.
}
\label{tab:canonicalization-results}
\small
\begin{tabular}{lcccccc}
\toprule
Metric
& Seed 0
& Seed 1
& Seed 2
& Seed 3
& Seed 4
& Mean $\pm$ SD \\
\midrule
\multicolumn{7}{l}{\textit{Canonicalization}} \\
Direct exact canonical match
& $0.9821$
& $1.0000$
& $1.0000$
& $1.0000$
& $1.0000$
& $0.9964 \pm 0.0071$ \\
\addlinespace

\multicolumn{7}{l}{\textit{Downstream multiplication}} \\
Training exact accuracy
& $0.9992$
& $1.0000$
& $0.9983$
& $0.9981$
& $1.0000$
& $0.9991 \pm 0.0008$ \\

Held-out exact accuracy
& $0.9869$
& $1.0000$
& $0.9983$
& $0.9981$
& $1.0000$
& $0.9967 \pm 0.0049$ \\

Held-out exact with recovery
& $0.9869$
& $1.0000$
& $0.9983$
& $0.9981$
& $1.0000$
& $0.9967 \pm 0.0049$ \\
\bottomrule
\end{tabular}
\end{table*}

\paragraph{Exact canonical matching.}
For each Galois orbit, we compare the learned canonical representative
of the held-out basis with that of a fixed training basis from the same
orbit. An exact match requires all nine binary entries to be equal.

The mean exact canonical match rate is \(0.9964\pm0.0071\). All \(56\)
orbits match for four seeds, while \(55\) of the \(56\) orbits match
for Seed~0. In total, learned canonicalization produces identical
representatives in \(279\) of the \(280\) orbit evaluations.

No approximate comparison is used in this result. The near perfect
match rate shows that the predictions of the learned transformation,
followed by the fixed lexicographic rule outside the model, almost
always produce the same canonical representative for a training basis
and a held-out basis from the same orbit.

\paragraph{Exact lookup and downstream multiplication.}
The downstream multiplication model reaches
\(0.9991\pm0.0008\) exact accuracy on the training bases and
\(0.9967\pm0.0049\) exact accuracy on the held-out bases. The held-out
result includes both multiplication errors and failures of exact
lookup. As defined in Section~\ref{sec:canonicalization}, a failure of exact
lookup causes all \(64\) multiplication examples from that held-out
basis to be counted as incorrect.

The exact canonical match rate and downstream lookup measure different
events. The former compares each held-out representative
with one fixed training representative from the same orbit. The latter instead searches all representatives obtained from the training
bases. A held-out representative may fail the fixed comparison but
still match another stored representative. The multiplication model
may also make errors after a successful lookup. For these reasons, downstream multiplication accuracy need not equal
the exact canonical match rate.

The near perfect downstream multiplication accuracy shows an important
difference between predicting an invariant directly and constructing
one from a learned transformation. Galois orbit identification asks the
model to map a single basis matrix to one of \(56\) arbitrary orbit
labels. These labels do not encode how bases in the same orbit are
related, and this prediction remains near chance on held-out bases.

By contrast, learned canonicalization asks the model
to predict the Frobenius action, which provides a common rule for moving between basis representations in every orbit and can be applied
repeatedly. A fixed procedure then converts the predicted orbit members
into a canonical representative and assigns an identifier through exact
matching. The model learns the structured relation between equivalent
representations, and the pipeline handles the final symbolic
assignment. This division explains why
learned canonicalization supports almost perfect downstream
multiplication even though direct Galois orbit identification does not
transfer.

\paragraph{Recovery using Hamming distance.}
We also evaluate a recovery procedure based on Hamming distance. When exact lookup fails, this procedure assigns the identifier of the stored training representative with the smallest Hamming distance. This rule yields
held-out exact accuracy of $0.9967 \pm 0.0049$, identical to the primary result for every seed. Because the primary accuracy is already close to perfect, there is little room for improvement. Within this limited margin, recovery based on Hamming distance produces no measurable gain.

Most importantly, the primary canonicalization and downstream
multiplication results do not depend on this recovery procedure. The
exact canonical match rate requires exact equality, and the primary
downstream evaluation counts every unresolved lookup failure as
incorrect.

\paragraph{Extension to $\mathbb F_{16}$.}
Appendix~\ref{app:f16} repeats the experiments on a fixed sample of
$50$ Galois orbits in $\mathbb F_{16}$. The extension preserves the
strong transfer of pairwise orbit recognition and the effectiveness of
an explicit Galois orbit label. However, the decomposed multiplication
pipeline, learned Galois action, and learned canonicalization are less
stable than in $\mathbb F_8$. Entrywise prediction of the Galois action
remains strong, but exact matrix prediction, repeated composition,
cycle closure, and primary downstream multiplication are markedly less
reliable. Ground truth canonicalization continues to support accurate
downstream multiplication. These results suggest that canonicalization remains effective in
$\mathbb F_{16}$ when the exact Galois structure is available. The
difficulty lies in predicting the Galois action accurately enough to
construct a reliable canonical representative.

\section{Discussion}
\label{sec:discussion}

\paragraph{Galois orbit label \(\ell_B\) versus learned canonical identifier
\(\widehat c_B\).}
In downstream multiplication, the learned canonical identifier
\(\widehat c_B\) serves the same functional role as the supplied Galois
orbit label \(\ell_B\). They both indicate which coordinate
multiplication table should be used. Their difference lies in how they are
obtained. The label \(\ell_B\) is defined from the known Galois orbit
and supplied directly to the multiplication model. Its availability
for a held-out basis is part of the experimental condition. By
contrast, \(\widehat c_B\) is constructed from predictions made on the
held-out basis. The model predicts the Frobenius action, after which
lexicographic selection and exact lookup are performed by fixed
procedures outside the model.

This distinction helps explain the contrast with Galois orbit
identification. Direct prediction of the Galois orbit label from a
held-out basis matrix remains near the chance level in
\(\mathbb F_8\). The canonicalization pipeline does not predict this
label directly. It predicts the structured transformation between
bases in the same Galois orbit and uses the resulting canonical
representative to retrieve an identifier through exact matching. In
\(\mathbb F_8\), this constructed identifier transfers almost
perfectly to held-out bases and supports near perfect downstream
multiplication.

\paragraph{Exact construction and its limitations.}
The \(\mathbb F_8\) result shows that a learned algebraic
transformation can be combined with a deterministic procedure to
construct an effective invariant. Generalization occurs in the
prediction of the Frobenius action on held-out basis matrices. The
subsequent construction of the representative and assignment of its
identifier are not additional learned predictions. The result does
not imply that the model discovers the canonicalization rule itself.

This construction is also sensitive to prediction errors. Exact lookup
requires equality of every matrix entry, and an incorrect prediction
may change the selected representative and prevent a match. In addition, unlike ground truth canonicalization, learned
canonicalization is not guaranteed to produce the same representative
for every basis in a Galois orbit. These limitations become more visible in
\(\mathbb F_{16}\), where entrywise prediction remains strong but exact
matrix prediction and canonical matching are less reliable. Learned
canonicalization still supports useful downstream multiplication, but
its performance is lower and more variable than in
\(\mathbb F_8\). The lexicographic rule is a fixed and reproducible
selection convention. It gives the selected matrix no additional
algebraic significance.

\paragraph{More direct mathematical information does not guarantee better learning.}
The basis conditioned multiplication experiment shows that providing
more mathematically informative input does not necessarily improve
learning. The Galois orbit label alone identifies the relevant
multiplication rule and supports perfect accuracy. The basis matrix
\(P_B\) contains additional information that uniquely determines the
basis and is algebraically related to the multiplication map. Even so, providing \(P_B\) without the orbit label does not support
reliable learning, and adding \(P_B\) to the orbit label makes the
result markedly less stable than using the label alone. The true
basis matrix also does not show a clear advantage over constant tokens
or shuffled matrices. Its mean held-out accuracy is higher, but the
variation is large and the results overlap across seeds.

\paragraph{Scope.}
The experiments study transfer across basis representations, not
generalization to unseen multiplication rules. Every held-out basis
belongs to a Galois orbit represented by training bases and shares its
coordinate multiplication table with those bases. The test set
introduces a basis representation that is not used for training, while
the corresponding multiplication rule remains present in the training
data. Accordingly, the conclusions concern transfer among
mathematically equivalent representations of a known rule.

\paragraph{Limitations.}
The evaluation of the learned Galois action uses a specific form of
data separation. A held-out basis matrix is excluded as a predictor
input during training, but it may appear as the prediction target for a
training basis from the same orbit. Hence, the results establish
generalization to held-out predictor inputs, not complete exclusion of
the held-out matrix from all roles in the training data.

The $\mathbb F_{16}$ experiments also show that the almost perfect
constructive results obtained in $\mathbb F_8$ do not transfer
unchanged to the larger field. Explicit Galois orbit labels and
pairwise orbit recognition remain effective, but the decomposed
pipeline, exact prediction of the Galois action, repeated composition,
and learned canonicalization become less reliable. Moreover, the
$\mathbb F_{16}$ experiments use a fixed sample of \(50\) Galois
orbits and do not provide exhaustive orbit coverage.

Finally, the experiments are limited to multiplication in small binary
finite fields, one model family, and controlled splits based on Galois
orbits. Several conditions also exhibit substantial variation across
experimental seeds. Future work could examine larger fields, broader
model families, training objectives that explicitly encourage
composition and cycle closure, and settings in which both the basis
representation and the underlying multiplication rule are unseen.

\section{Conclusion}
\label{sec:conclusion}
We studied whether neural models can generalize multiplication in finite fields
across mathematically equivalent basis representations. 
In $\mathbb F_8$, an explicit Galois orbit label supports perfect
transfer because training and held-out bases from the same Galois orbit
share a coordinate multiplication table. Galois orbit identification fails to generalize to held-out basis representations,
while pairwise orbit recognition transfers successfully.

In the main $\mathbb F_8$ experiments, the underlying algebraic
operations can also be learned constructively. Under strict separation
of the transformations used for training and evaluation, the
decomposed multiplication pipeline combines a learned matrix vector
module with a learned reference basis multiplication module and
achieves strong held-out performance. The learned Galois action
generalizes almost perfectly to held-out predictor inputs and preserves
two-step composition and three-step cycle closure. It then supports
almost exact canonicalization and near perfect downstream multiplication
under the primary exact lookup procedure, without recovery based on
Hamming distance.

The $\mathbb F_{16}$ results preserve several qualitative patterns from
$\mathbb F_8$, but also show that the learned constructive procedures
become less reliable as the setting grows. The continued success of
ground truth canonicalization suggests that accurate transformation
learning is critical to extending these procedures beyond
$\mathbb F_8$.

In conclusion, the results show that learning an algebraic operation
and transferring it across equivalent representations are distinct
problems. Successful transfer depends on how the equivalence structure
is presented and used, and the comparison between $\mathbb F_8$ and
$\mathbb F_{16}$ shows that this success does not automatically extend
to more demanding settings.

\bibliographystyle{apalike}
\bibliography{references}
\clearpage

\appendix
\section{Proofs of the Algebraic Results}
\label{app:proofs}
We provide the proofs of the lemma and theorem in section~\ref{sec:mathematical-framework} here.
\subsection{Proof of Lemma~\ref{lem:free-galois-action}}

\begin{proof}
Let
\[
B=(b_0,\ldots,b_{n-1})
\]
be an ordered $\mathbb{F}_2$-basis of
$F=\mathbb{F}_{2^n}$. Suppose that
\[
\sigma^k(B)=B.
\]
Because the basis is ordered, this equality implies
\[
\sigma^k(b_i)=b_i
\qquad
\text{for every }i=0,\ldots,n-1.
\]

Every element $x\in F$ has a unique expansion
\[
x=\sum_{i=0}^{n-1}u_i b_i,
\qquad
u_i\in\mathbb{F}_2.
\]
Since $\sigma^k$ fixes $\mathbb{F}_2$ pointwise, we obtain
\[
\sigma^k(x)
=
\sum_{i=0}^{n-1}u_i\sigma^k(b_i)
=
\sum_{i=0}^{n-1}u_i b_i
=
x.
\]
Thus, $\sigma^k$ fixes every element of $\mathbb{F}$, and hence
\[
\sigma^k=\operatorname{id}_{F}.
\]

The Frobenius automorphism $\sigma(x)=x^2$ has order $n$ in
\[
\operatorname{Gal}(\mathbb{F}_{2^n}/\mathbb{F}_2),
\]
so $\sigma^k=\operatorname{id}_{F}$ if and only if
\[
k\equiv0\pmod n.
\]
The converse is immediate. Therefore, the stabilizer of every ordered
basis under the Galois action is trivial, and every Galois orbit has
exactly
\[
\left|
\operatorname{Gal}(\mathbb{F}_{2^n}/\mathbb{F}_2)
\right|
=n
\]
elements.
\end{proof}

\subsection{Proof of Theorem~\ref{thm:galois-multiplication}}

\begin{proof}
Let
\[
B=(b_0,\ldots,b_{n-1})
\qquad\text{and}\qquad
C=(c_0,\ldots,c_{n-1})
\]
be ordered $\mathbb{F}_2$-bases of
$F=\mathbb{F}_{2^n}$. Define the isomorphisms
\[
\phi_B,\phi_C:\mathbb{F}_2^n\longrightarrow F
\]
by
\[
\phi_B(u_0,\ldots,u_{n-1})
=
\sum_{i=0}^{n-1}u_i b_i,
\qquad
\phi_C(u_0,\ldots,u_{n-1})
=
\sum_{i=0}^{n-1}u_i c_i.
\]
By the definition of \(m_B\) in
Section~\ref{sec:basis-multiplication}, for \(u,v\in\mathbb F_2^n\), we have
\[
m_B(u,v)
=
\phi_B^{-1}
\left(
\phi_B(u)\phi_B(v)
\right),
\] and similarly for $m_C$.

First, suppose that $B$ and $C$ belong to the same Galois orbit. Then
\[
C=\sigma^r(B)
\]
for some $r\in\{0,\ldots,n-1\}$. Since $\sigma^r$ fixes
$\mathbb{F}_2$ pointwise,
\[
\phi_C=\sigma^r\circ\phi_B.
\]
Consequently,
\[
\phi_C^{-1}
=
\phi_B^{-1}\circ\sigma^{-r}.
\]
For every $u,v\in\mathbb{F}_2^n$, we then have
\begin{align*}
m_C(u,v)
&=
\phi_C^{-1}\bigl(\phi_C(u)\phi_C(v)\bigr)\\
&=
\phi_B^{-1}\circ\sigma^{-r}
\bigl(\sigma^r(\phi_B(u))\sigma^r(\phi_B(v))\bigr)\\
&=
\phi_B^{-1}\circ\sigma^{-r}
\bigl(\sigma^r(\phi_B(u)\phi_B(v))\bigr)\\
&=
\phi_B^{-1}\bigl(\phi_B(u)\phi_B(v)\bigr)\\
&=
m_B(u,v).
\end{align*}
Hence $m_B=m_C$.

Conversely, suppose that
\[
m_B=m_C.
\]
Define
\[
T=\phi_C\circ\phi_B^{-1}:
F \longrightarrow F.
\]
Because $\phi_B$ and $\phi_C$ are
$\mathbb{F}_2$-linear isomorphisms, $T$ is an
$\mathbb{F}_2$-linear bijection. We claim that $T$ also preserves
multiplication.

Let $a,b\in F$, and choose $u,v\in\mathbb{F}_2^n$ such that
\[
a=\phi_B(u),
\qquad
b=\phi_B(v).
\]
Then
\begin{align*}
T(ab)
&=
\phi_C\bigl(\phi_B^{-1}(ab)\bigr)\\
&=
\phi_C\bigl(m_B(u,v)\bigr)\\
&=
\phi_C\bigl(m_C(u,v)\bigr)\\
&=
\phi_C(u)\phi_C(v)\\
&=
T(a)T(b).
\end{align*}
Thus, $T$ is an $\mathbb{F}_2$-algebra automorphism of
$F$. Since
\[
\operatorname{Gal}
(\mathbb{F}_{2^n}/\mathbb{F}_2)
=
\langle\sigma\rangle,
\]
there exists $r\in\{0,\ldots,n-1\}$ such that
\[
T=\sigma^r.
\]
Therefore,
\[
\phi_C
=
T\circ\phi_B
=
\sigma^r\circ\phi_B.
\]

Let $e_i$ denote the $i$th standard basis vector of
$\mathbb{F}_2^n$. Applying the preceding identity to $e_i$ gives
\[
c_i
=
\phi_C(e_i)
=
\sigma^r\bigl(\phi_B(e_i)\bigr)
=
\sigma^r(b_i)
\]
for every $i=0,\ldots,n-1$. Hence
\[
C=\sigma^r(B),
\]
so $B$ and $C$ belong to the same Galois orbit.
\end{proof}

\section{Additional Experiments in $\mathbb{F}_{16}$}
\label{app:f16}
\subsection{Experimental Setting}
\label{app:f16-setting}

We extend the $\mathbb F_8$ experiments to
$\mathbb F_{16}$. The $20{,}160$ ordered
$\mathbb F_2$-bases of $\mathbb F_{16}$ form $5{,}040$ Galois orbits,
each containing four bases. Because exhaustive training over all
orbits is expensive, we
sample $50$ distinct orbits.

The orbit sample is generated once using a fixed data seed and is
shared across all tasks, conditions, and experimental seeds. Within
each sampled orbit, three bases are used for training and the
remaining basis is held out for testing. This produces $150$ training
bases and $50$ held-out bases. In all, the sampled experiment uses
$200$ of the $20{,}160$ ordered bases of $\mathbb F_{16}$.

Unlike the $\mathbb F_8$ experiments, the sampled orbits and the
training and test split remain fixed across experimental seeds.
Variation across seeds reflects model initialization and
stochastic optimization, not changes in the sampled orbits or the
held-out bases.

For each basis, we enumerate all
\[
16^2=256
\]
ordered operand pairs. Each basis contributes its complete
coordinate multiplication table. The resulting multiplication
datasets contain
\[
150\times256=38{,}400
\]
training examples and
\[
50\times256=12{,}800
\]
held-out examples.

Unless otherwise stated, results are reported as the mean and
population standard deviation across five experimental seeds. The
condition without a Galois orbit label or basis matrix was run for
three seeds and all other reported conditions use five.

\subsection{Basis conditioned multiplication}
\label{app:f16-basis-conditioning}

We first repeat the six basis-conditioning conditions from the
$\mathbb F_8$ experiment. Table~\ref{tab:f16-basis-conditioning}
reports the final training and held-out exact accuracy.

\begin{table*}[t]
\centering
\caption{
Final exact accuracy for basis conditioned multiplication in
$\mathbb F_{16}$. Values are means and population standard deviations
across experimental seeds.
}
\label{tab:f16-basis-conditioning}
\small
\begin{tabular}{lccc}
\toprule
Condition
& Seeds
& Train exact
& Held-out exact \\
\midrule
Orbit label, no $P_B$
& $5$
& $0.9791 \pm 0.0301$
& $0.9791 \pm 0.0301$ \\

$P_B$, no orbit label
& $5$
& $0.5805 \pm 0.2474$
& $0.2576 \pm 0.0439$ \\

No orbit label, no $P_B$
& $3$
& $0.2350 \pm 0.0006$
& $0.2350 \pm 0.0006$ \\

Orbit label $+\,P_B$
& $5$
& $0.5827 \pm 0.3448$
& $0.5803 \pm 0.3441$ \\

Orbit label $+$ constant tokens
& $5$
& $0.7118 \pm 0.1274$
& $0.7118 \pm 0.1274$ \\

Orbit label $+$ shuffled matrix
& $5$
& $0.7622 \pm 0.2694$
& $0.7623 \pm 0.2692$ \\

\bottomrule
\end{tabular}
\end{table*}

\paragraph{Orbit label and basis information.}
Providing the Galois orbit label without \(P_B\) yields training and
held-out exact accuracy of \(0.9791 \pm 0.0301\). As in
\(\mathbb F_8\), the label identifies a multiplication table that is
already represented by the three training bases from the same orbit.
Performance is close to perfect, although the final accuracy varies
across seeds under the fixed training schedule.

When neither the Galois orbit label nor \(P_B\) is provided, exact
accuracy remains at the floor of \(0.2350 \pm 0.0006\). The model
receives only the coordinate operand pair, which can correspond to
different outputs across the sampled Galois orbits.

Providing \(P_B\) without the Galois orbit label produces a different
pattern. Training exact accuracy rises to \(0.5805 \pm 0.2474\), but
held-out exact accuracy remains much lower at
\(0.2576 \pm 0.0439\). Four of the five seeds perform considerably
above the floor on the training bases, but this improvement does not
transfer to held-out basis representations. This suggests that the model
can partially fit the relationship between \(P_B\) and the
multiplication targets for the training bases without learning a rule
that extends reliably to a new basis from the same orbit.

\paragraph{Additional inputs with the Galois orbit label.}
The remaining three conditions add \(16\) tokens to the Galois orbit
label and operands. These tokens contain the entries of \(P_B\), fixed
constant values, or the entries of a shuffled basis matrix. In the
shuffled condition, a derangement assigns the \(200\) legitimate basis
matrices to different bases. The matrix distribution remains unchanged,
but the correct relationship between each basis and its matrix is
removed.

The three conditions differ widely across seeds. The true
\(P_B\) condition has the lowest mean held-out accuracy and the largest
variation, with outcomes ranging from the floor to almost perfect
accuracy. The constant tokens produce more consistent intermediate to
high accuracy, while the shuffled matrices achieve the highest mean but
remain sensitive to the experimental seed. Thus, the true basis matrix
does not provide an advantage over either control in the sampled
\(\mathbb F_{16}\) experiment.

\paragraph{Seed-wise learning dynamics.}
Figure~\ref{fig:f16-length-matched-curves} shows the trajectories of
held-out exact accuracy for the three conditions with equal sequence
length. The true and shuffled matrix conditions produce highly variable outcomes across seeds, while the constant token condition is
more concentrated at intermediate to high accuracy.

\begin{figure*}[t]
    \centering
    \includegraphics[width=\textwidth]
    {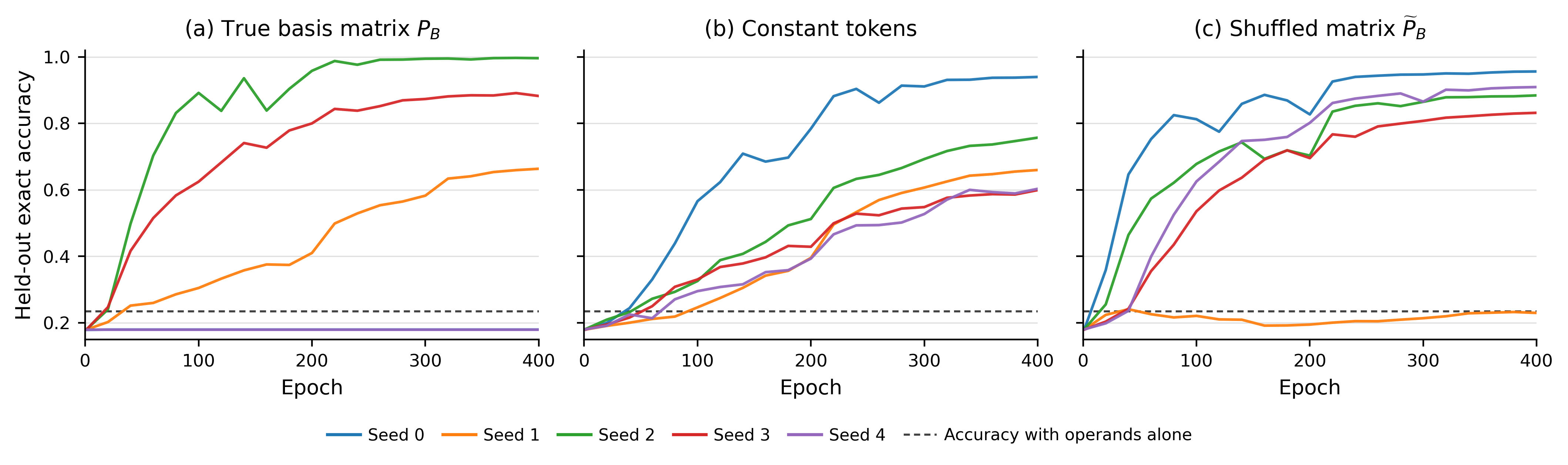}
    \caption{
    Trajectories of held-out exact accuracy for the three conditions with
equal sequence length in $\mathbb F_{16}$.
The dotted horizontal line marks the accuracy obtained using the
operands alone, \(0.2350\).
    }
    \label{fig:f16-length-matched-curves}
\end{figure*}

\paragraph{Interpretation.}
The true basis matrix has the lowest mean held-out accuracy. The constant tokens and
shuffled matrices have higher mean accuracies. However, the results
vary widely across the five seeds, so these mean differences do not
establish that either control is consistently better than the true
basis matrix. Successful runs are not unique
to the true \(P_B\) condition, and the results provide no evidence that
its correct algebraic relationship with the multiplication task gives
a stable advantage under the present model and training schedule.
Additional input positions can affect learning even when
they contain no information about the basis or contain a matrix that is
incorrectly associated with it.

The condition using \(P_B\) without the Galois orbit label also differs
from its counterpart in \(\mathbb F_8\). In \(\mathbb F_8\), training
and held-out accuracy both remain near the floor. In
\(\mathbb F_{16}\), the model partially fits the training bases, but
this improvement does not transfer to held-out bases. This pattern
indicates stronger fitting to the individual training bases without a
corresponding improvement on new bases from the same orbits. However,
the two field settings also differ in orbit coverage, dataset size,
split variation, and training schedule, so the difference cannot be
attributed solely to the field or matrix dimension.

\subsection{Orbit recognition}
\label{app:f16-orbit-recognition}
We next evaluate whether Galois orbit structure transfers to held-out
basis representations in $\mathbb F_{16}$. As in the $\mathbb F_8$
experiments, we compare Galois orbit identification with
pairwise orbit recognition. Table~\ref{tab:f16-orbit-recognition}
reports the final results.

\begin{table*}[t]
\centering
\caption{
Galois orbit identification and pairwise orbit recognition in
$\mathbb F_{16}$. Values are means and population standard deviations
across five experimental seeds.
}
\label{tab:f16-orbit-recognition}
\small
\begin{tabular}{lcccc}
\toprule
Task
& Overall train
& Overall held-out
& Positive held-out
& Negative held-out \\
\midrule
Galois orbit identification
& $1.0000 \pm 0.0000$
& $0.1720 \pm 0.0299$
& ---
& --- \\

Pairwise orbit recognition
& $0.9997 \pm 0.0007$
& $0.9093 \pm 0.0274$
& $0.9667 \pm 0.0193$
& $0.8520 \pm 0.0396$ \\
\bottomrule
\end{tabular}
\end{table*}

\paragraph{Galois orbit identification.}
The model reaches perfect training accuracy but obtains held-out
accuracy of \(0.1720 \pm 0.0299\). Although this is above
the random guessing level of \(1/50=0.0200\), it remains far below the
training result. The model shows some transfer to the
held-out basis from each sampled orbit, but it does not identify these
bases reliably.

The task uses the same fixed set of \(50\) sampled orbits for training
and evaluation. It tests a new basis representation from each orbit,
not bases belonging to previously unseen orbits. The huge gap between
training and held-out accuracy reflects poor transfer to the
held-out representation, not the introduction of new classes.

\paragraph{Pairwise orbit recognition.}
The pairwise task uses balanced sets of unique ordered pairs. The
training set contains \(300\) positive and \(300\) negative pairs,
while the held-out set contains \(150\) pairs of each class. Every
held-out pair includes the held-out basis from one of the sampled
orbits.

The model obtains overall accuracy of \(0.9997 \pm 0.0007\) on the
training pairs and \(0.9093 \pm 0.0274\) on the held-out pairs. On the
held-out set, accuracy is \(0.9667 \pm 0.0193\) for positive pairs and
\(0.8520 \pm 0.0396\) for negative pairs. This indicates that the model is more accurate on pairs from the same
orbit than on pairs from different orbits.

\paragraph{Learning dynamics.}
Figure~\ref{fig:f16-orbit-recognition-curves} shows the corresponding
training trajectories. Galois orbit identification reaches perfect training accuracy rapidly,
while held-out accuracy stabilizes between approximately \(0.12\) and
\(0.20\) across seeds. Pairwise recognition also approaches perfect
training accuracy, but maintains held-out accuracy near \(0.90\).
After the initial stage of training, accuracy on positive pairs remains
consistently higher than accuracy on negative pairs.

\begin{figure*}[t]
    \centering
    \includegraphics[width=\textwidth]
    {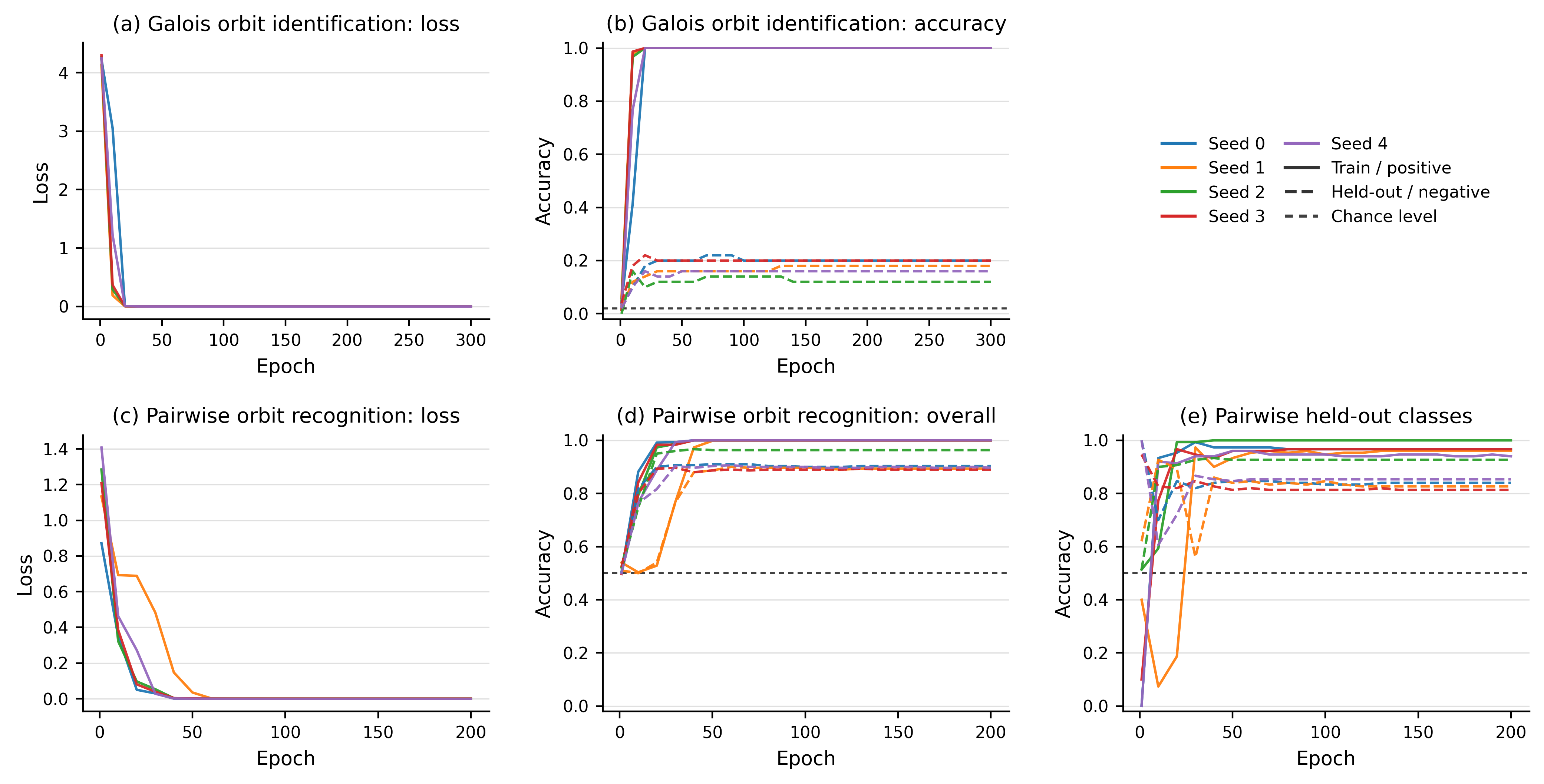}
    \caption{
    Learning curves for Galois orbit identification and pairwise orbit
recognition in $\mathbb F_{16}$. Solid curves show training
accuracy and dashed curves show held-out accuracy.
    }
    \label{fig:f16-orbit-recognition-curves}
\end{figure*}

\paragraph{Comparison with $\mathbb F_8$.}
Pairwise orbit recognition transfers strongly in both fields. Mean
held-out accuracy is approximately \(0.87\) in \(\mathbb F_8\) and
\(0.91\) in \(\mathbb F_{16}\), with accuracy on positive pairs
exceeding accuracy on negative pairs in both settings.

Galois orbit identification differs more clearly between the two
fields. Accuracy in \(\mathbb F_8\) remains near the chance level,
whereas the \(\mathbb F_{16}\) result is above chance. Because the
\(\mathbb F_{16}\) experiment uses a fixed sample of \(50\) orbits,
this comparison should be interpreted within the sampled setting.
Regardless, the result shows that some information about the orbit
labels transfers to held-out basis representations in this experiment.

\subsection{Decomposed multiplication pipeline}
\label{app:f16-decomposed-pipeline}
We apply the decomposed multiplication pipeline of
Section~\ref{sec:decomposed-pipeline} to \(\mathbb F_{16}\) using the
same decomposition and evaluation procedure as in the
\(\mathbb F_8\) experiment. After duplicate matrices are removed, we
exclude every matrix equal to \(P_B\) or \(P_B^{-1}\) for any held-out
basis. The resulting training set for the matrix vector module contains
\(297\) distinct transformations. So none of the transformations
required by the held-out bases is used to train this module.

Table~\ref{tab:f16-decomposed-results} reports the final held-out
multiplication accuracy for each seed.

\begin{table}[t]
\centering
\caption{
Held-out multiplication accuracy of the decomposed multiplication
pipeline in $\mathbb F_{16}$. The final row reports the mean and
population standard deviation across five experimental seeds.
}
\label{tab:f16-decomposed-results}
\small
\begin{tabular}{lcc}
\toprule
Seed
& Bit accuracy
& Exact accuracy \\
\midrule
$0$ & $0.9541$ & $0.9008$ \\
$1$ & $0.6002$ & $0.2338$ \\
$2$ & $0.5896$ & $0.1800$ \\
$3$ & $0.5898$ & $0.1798$ \\
$4$ & $0.5000$ & $0.0586$ \\
\midrule
Mean $\pm$ population SD
& $0.6467 \pm 0.1579$
& $0.3106 \pm 0.3006$ \\
\bottomrule
\end{tabular}
\end{table}

\paragraph{Held-out multiplication.}
Performance varies widely across seeds. Seed~0 achieves exact
accuracy of \(0.9008\). This demonstrates that high held-out accuracy is possible, but the much
lower results for the other seeds show that this outcome is not
reliable across runs. The remaining seeds obtain exact
accuracy between \(0.0586\) and \(0.2338\). The mean exact accuracy of
\(0.3106 \pm 0.3006\) reflects widely different outcomes
across seeds, not a stable intermediate level of performance.

\paragraph{Module-level learning dynamics.}
Figure~\ref{fig:f16-decomposed-module-curves} shows that the reference basis
multiplication module reaches high training accuracy for all five
seeds. In contrast, the matrix vector module is highly sensitive to the
experimental seed. Seed~0 approaches perfect training accuracy,
Seed~1 learns the operation only partially, and Seeds~2--4 remain at
low accuracy. The performance of the matrix vector module closely
matches the final held-out multiplication results.

\begin{figure*}[t]
    \centering
    \includegraphics[width=\textwidth]
    {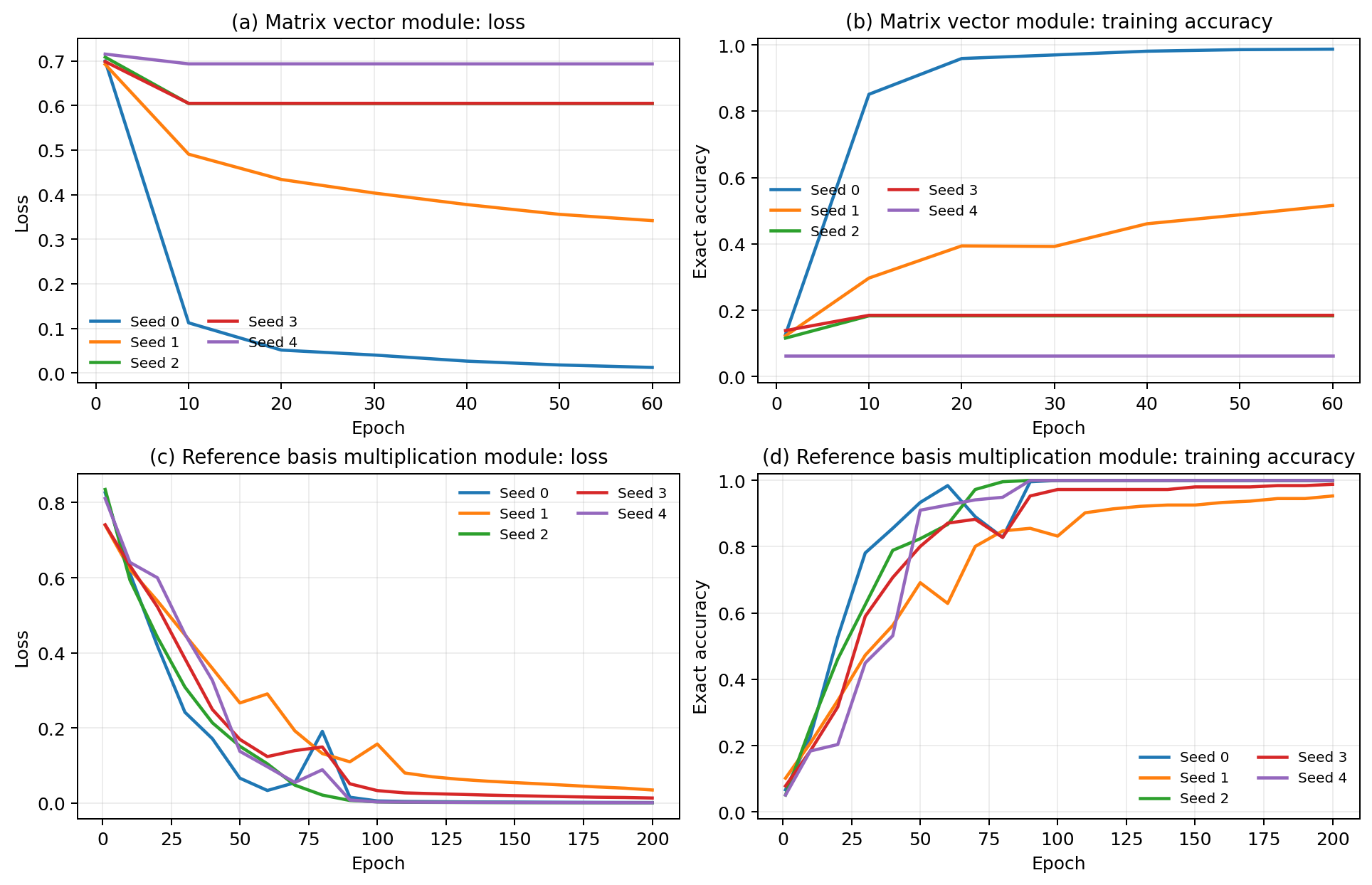}
    \caption{
    Training histories of the two learned modules in the
$\mathbb F_{16}$ decomposed multiplication pipeline.
    }
    \label{fig:f16-decomposed-module-curves}
\end{figure*}

\paragraph{Comparison with $\mathbb F_8$.}
The $\mathbb F_{16}$ results are much worse than the
corresponding $\mathbb F_8$ results. The reference basis multiplication
module is learned reliably in both settings, but the matrix vector
module becomes the principal limitation in $\mathbb F_{16}$. Because
this module is used to transform both operands and to convert the final
product back to the held-out basis, its errors can spread through
the complete pipeline.

\subsection{Learning the Galois action}
\label{app:f16-galois-action}

We repeat the held-out-input evaluation of
Section~\ref{sec:learning-galois-action} in $\mathbb F_{16}$. The learned
transformation $\widehat{\sigma}$ is trained on the $150$ training
bases and evaluated on the $50$ held-out basis matrices.
Because the Galois group $\text{Gal}(\mathbb F_{16}/\mathbb F_2)$ has order
four, evaluation includes one-step prediction, two- and three-step
composition, and four-step cycle closure.

\begin{table}[t]
\centering
\caption{
Performance of the learned transformation $\widehat{\sigma}$ on
held-out inputs in $\mathbb F_{16}$. Values are the mean and population
standard deviation across five experimental seeds.
}
\label{tab:f16-galois-action-results}
\small
\begin{tabular}{lc}
\toprule
Metric & Accuracy \\
\midrule
One-step bit accuracy
& $0.9295 \pm 0.0096$ \\

One-step exact accuracy
& $0.3280 \pm 0.0722$ \\

Two-step composition exact accuracy
& $0.3480 \pm 0.0711$ \\

Three-step composition exact accuracy
& $0.3880 \pm 0.0588$ \\

Four-step cycle-closure accuracy
& $0.4280 \pm 0.0412$ \\
\bottomrule
\end{tabular}
\end{table}

\paragraph{Held-out prediction.}
Bit accuracy after one application is consistently high, reaching
\(0.9295 \pm 0.0096\). Exact prediction of the complete \(4\times4\)
matrix is much more difficult, with accuracy of
\(0.3280 \pm 0.0722\). Therefore, most matrix entries are predicted
correctly, but errors in a small number of entries frequently prevent
the complete matrix from being recovered exactly.

Exact accuracy increases from \(0.3280\) after one application to
\(0.3480\) after two, \(0.3880\) after three, and \(0.4280\) after
four. These metrics are evaluated independently at each step, so the
increase indicates that some trajectories reach the correct orbit
position after earlier errors. Even so, fewer than half of the
held-out inputs return exactly to the original basis after four
applications. Hence, the learned transformation does not reliably
reproduce the full Galois action.

\paragraph{Training and held-out behavior.}
Figure~\ref{fig:f16-galois-action-results} shows the training histories
and final held-out results. Training exact accuracy reaches \(1.0\) for
every seed within approximately \(20\) epochs. In contrast, final exact
accuracy after one step ranges from \(0.2200\) to \(0.4400\) on
held-out inputs, despite consistently high bit accuracy. The result reflects a large generalization gap, not a failure to
fit the training transformations.

\begin{figure*}[t]
    \centering
    \includegraphics[width=\textwidth]
    {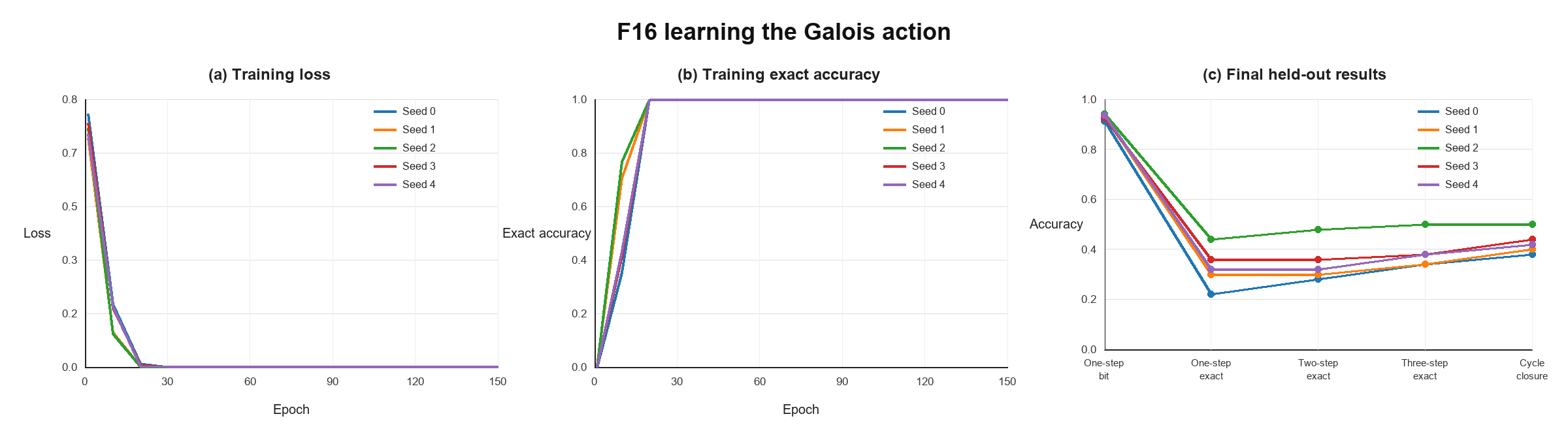}
    \caption{
    Training histories and final held-out results for the learned Galois
action in \(\mathbb F_{16}\). 
    }
    \label{fig:f16-galois-action-results}
\end{figure*}

\paragraph{Comparison with $\mathbb F_8$.}
The learned transformation generalizes less reliably in
\(\mathbb F_{16}\) than in the corresponding \(\mathbb F_8\)
experiment. In \(\mathbb F_8\), prediction after one step,
repeated composition, and cycle closure after three applications are
all nearly exact. In \(\mathbb F_{16}\), prediction of individual
matrix entries remains strong, but exact matrix prediction and cycle
closure are much weaker. The \(\mathbb F_{16}\) results capture much of the Galois action at the
level of individual matrix entries, but they do not reproduce the
almost exact held-out transformations observed in \(\mathbb F_8\).

\subsection{Canonicalization and downstream multiplication}
\label{app:f16-canonicalization}

We repeat the canonicalization and downstream multiplication experiment
of Section~\ref{sec:canonicalization} in \(\mathbb F_{16}\). We use the
same definitions of the ground truth canonical representative \(C(B)\)
and the learned canonical representative \(\widehat C(B)\), along
with the same primary exact lookup rule. 

\paragraph{Exact canonical matching and downstream multiplication.}
Table~\ref{tab:f16-canonicalization-results} reports the exact
canonical match rate and downstream multiplication accuracy. Under
learned canonicalization, the exact canonical match rate is
\(0.6720 \pm 0.0652\). The downstream multiplication model reaches
training exact accuracy of \(0.9065 \pm 0.1367\) and primary held-out
exact accuracy of \(0.7253 \pm 0.1382\).

\begin{table*}[t]
\centering
\caption{
Canonicalization and downstream multiplication in
\(\mathbb F_{16}\). Values are means and population standard deviations
across five experimental seeds.
}
\label{tab:f16-canonicalization-results}
\small
\begin{tabular}{lccc}
\toprule
Condition
& Exact canonical match
& Downstream train exact
& Downstream held-out exact \\
\midrule
Learned canonicalization
& $0.6720 \pm 0.0652$
& $0.9065 \pm 0.1367$
& $0.7253 \pm 0.1382$ \\

Ground truth canonicalization
& $1.0000 \pm 0.0000$
& $0.9728 \pm 0.0221$
& $0.9728 \pm 0.0221$ \\
\bottomrule
\end{tabular}
\end{table*}

As defined in Section~\ref{sec:canonicalization}, the exact canonical
match rate compares each held-out representative with one fixed
training representative from the same orbit, whereas downstream lookup
searches all representatives obtained from the training bases.
As a result, downstream held-out accuracy can exceed the exact canonical
match rate. Any basis for which exact lookup fails contributes no
correct downstream predictions.

\paragraph{Recovery using Hamming distance.}
When exact lookup fails, we also evaluate a recovery procedure that
selects the stored representative with the smallest Hamming distance.
This procedure increases downstream held-out exact accuracy to
\(0.8014 \pm 0.1269\), compared with \(0.7253 \pm 0.1382\) under the
primary exact lookup procedure. This result indicates that
approximate matching can recover useful canonical identifiers when an exact representative is not found. Recovery based on Hamming
distance is evaluated separately and is not used in the primary
results. 

\paragraph{Ground truth canonicalization.}
Ground truth canonicalization produces an exact canonical match rate of
\(1.0000\). With these canonical representatives, downstream
multiplication reaches exact accuracy of \(0.9728 \pm 0.0221\) on both
the training and held-out sets. The remaining error is due to
the downstream multiplication model, not canonical lookup. This
condition provides a reference for the performance attainable when
canonicalization is exact, but it is not a result of learned
canonicalization.

\paragraph{Behavior across seeds.}
Figure~\ref{fig:f16-canonicalization-results} shows greater variation
under learned canonicalization than under ground truth
canonicalization. Seed~4 has the lowest learned downstream training
accuracy and the lowest primary held-out accuracy. Recovery using
Hamming distance improves every learned canonicalization run, but
ground truth canonicalization remains more accurate for all five seeds.

\begin{figure*}[t]
    \centering
    \includegraphics[width=\textwidth]{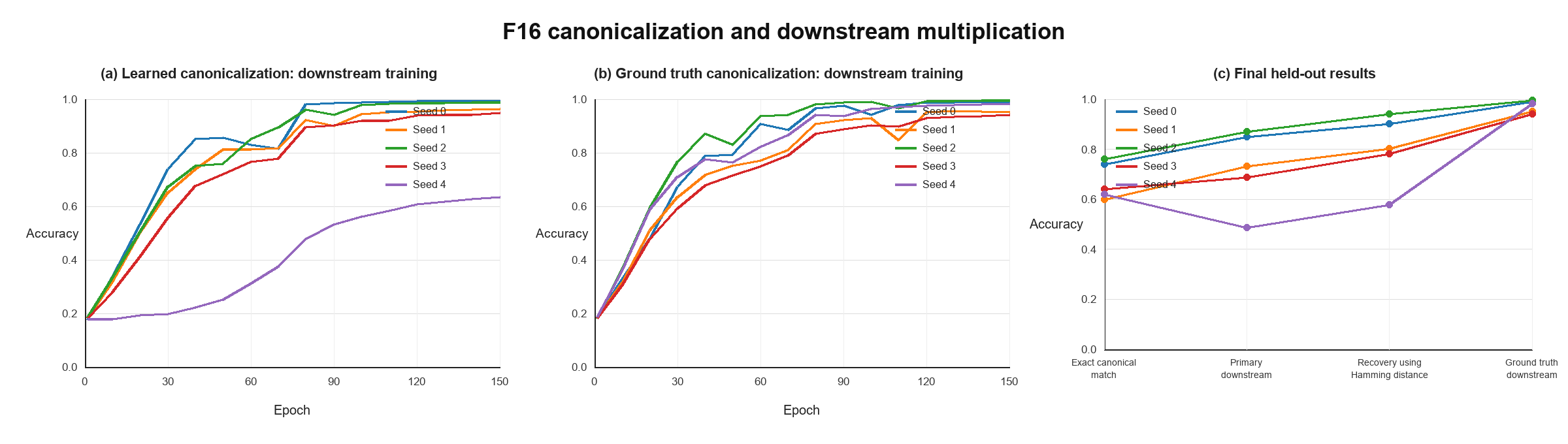}
    \caption{
    Training histories and held-out results for canonicalization and
downstream multiplication in \(\mathbb F_{16}\). Results are shown for
learned and ground truth canonicalization across five experimental
seeds.
    }
    \label{fig:f16-canonicalization-results}
\end{figure*}

\paragraph{Comparison with $\mathbb F_8$.}
Canonicalization is less reliable in \(\mathbb F_{16}\) than in the
corresponding \(\mathbb F_8\) experiment. In \(\mathbb F_8\), learned
canonical representatives almost always match exactly, and primary
downstream multiplication is nearly perfect. In \(\mathbb F_{16}\),
the learned representative of a held-out basis matches that of the
fixed training basis from the same orbit much less often. Primary
downstream multiplication is also lower and more variable.

The strong performance under ground truth canonicalization shows that
exact canonical representatives can still support accurate downstream
multiplication in \(\mathbb F_{16}\). The main limitation is constructing these representatives reliably from the learned Galois
action, although errors in the downstream multiplication model also
contribute to the final result.

\section{$\mathbb F_8$ Additional Learning Curves}

\subsection{Basis conditioned multiplication}
\label{app:f8-basis-curves}
Figure~\ref{fig:f8-appendix-basis-baselines} shows the complete
training histories for the three baseline conditions. The condition
using the Galois orbit label without \(P_B\) converges to perfect
training and held-out exact accuracy across seeds. When neither the
label nor \(P_B\) is provided, both training and held-out accuracy
remain at the floor. Providing \(P_B\) without the label also fails to
produce stable improvement, although some seeds show temporary
increases during training.

\begin{figure*}[t]
    \centering
    \includegraphics[width=\textwidth]
    {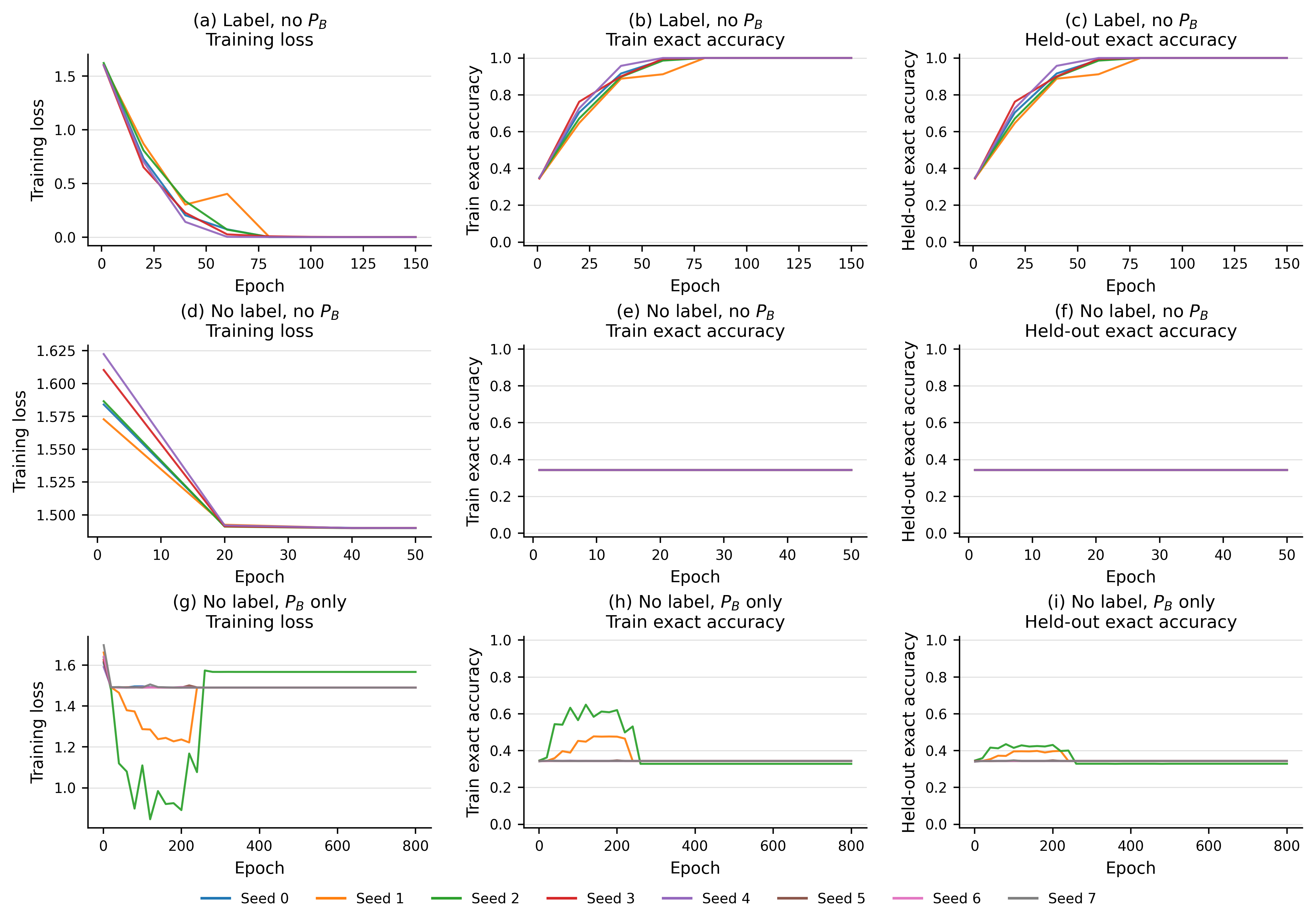}
    \caption{
    Learning curves for multiplication using the Galois orbit label alone,
neither the label nor \(P_B\), and \(P_B\) alone. Each color denotes one
experimental seed. Curves with identical values may overlap.
    }
    \label{fig:f8-appendix-basis-baselines}
\end{figure*}

Figure~\ref{fig:f8-appendix-length-matched} reports the complete
training histories for the three conditions with equal sequence length.
These conditions provide the true basis matrix, constant tokens, or a
shuffled matrix, respectively. Their loss, training accuracy, and
held-out accuracy trajectories all vary significantly across seeds.
Several runs remain at the floor and overlap visually. The corresponding final accuracies are reported in
Table~\ref{tab:basis-conditioning-results}.

\begin{figure*}[t]
    \centering
    \includegraphics[width=\textwidth]
    {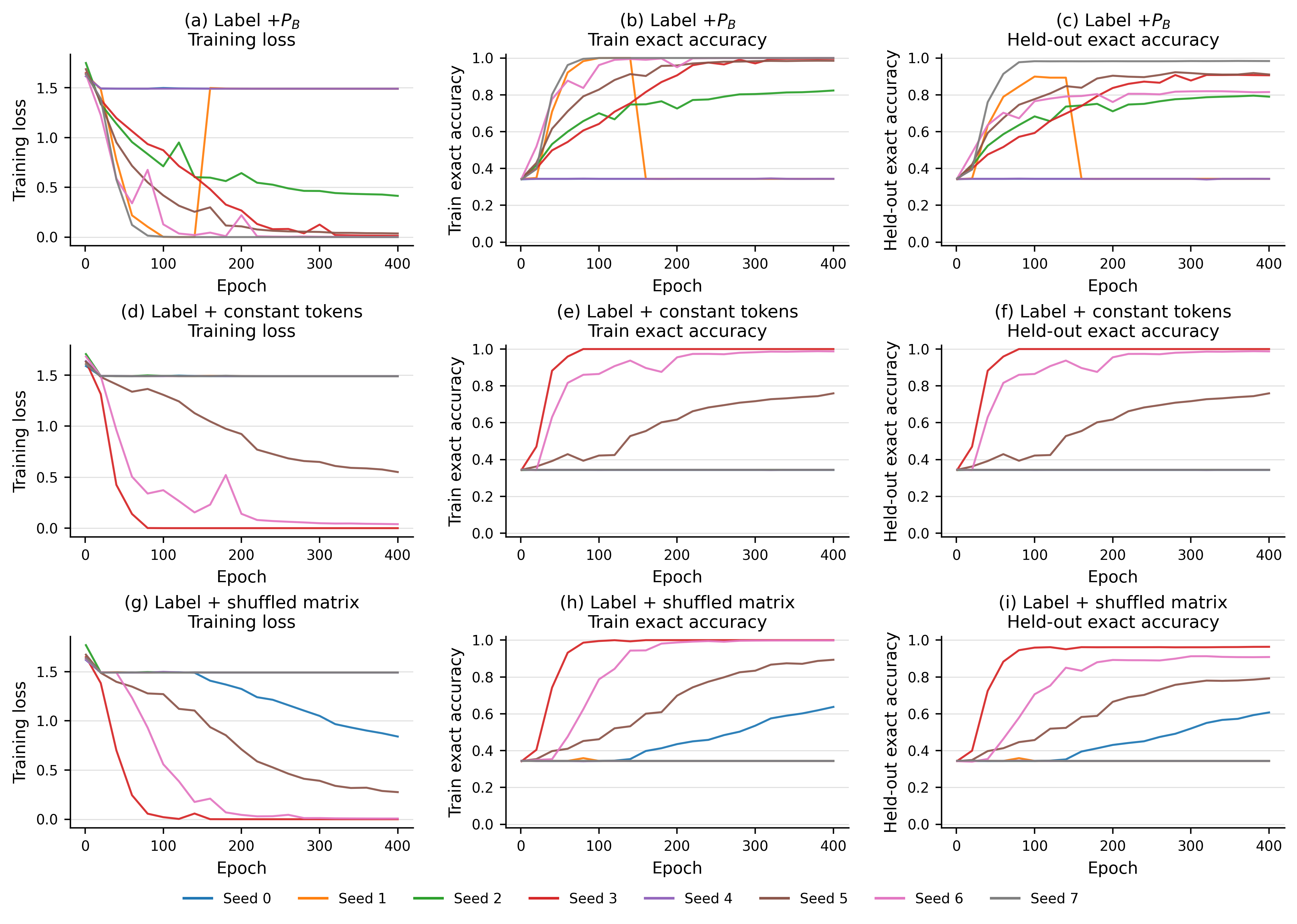}
    \caption{
    Complete learning curves for the three conditions with equal sequence
length.
    }
    \label{fig:f8-appendix-length-matched}
\end{figure*}

\subsection{Orbit recognition}
\label{app:f8-orbit-curves}

Figure~\ref{fig:f8-appendix-orbit-recognition} compares the training
dynamics of Galois orbit identification and pairwise orbit recognition.
Galois orbit identification rapidly reaches perfect training accuracy,
while held-out accuracy remains near the \(1/56\) chance level.
Pairwise orbit recognition instead achieves high held-out accuracy on
the balanced pair set, with consistently higher accuracy on positive
pairs than on negative pairs. Final results for both tasks are reported in
Table~\ref{tab:orbit-recognition-results}.

\begin{figure*}[t]
    \centering
    \includegraphics[width=\textwidth]
    {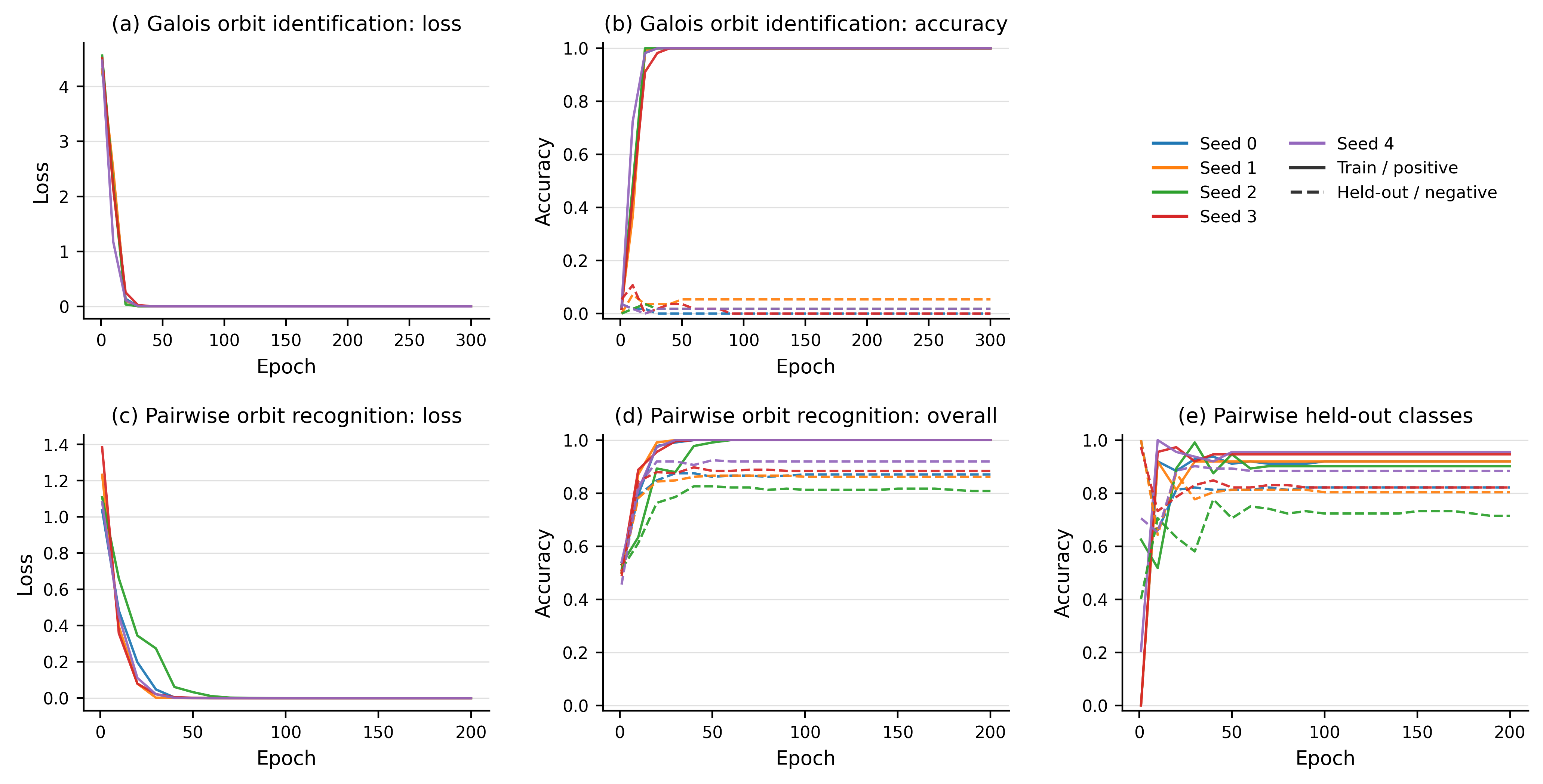}
    \caption{
    Learning curves for Galois orbit identification and pairwise orbit
recognition. Solid and dashed
lines distinguish the quantities indicated in the legend.
    }
    \label{fig:f8-appendix-orbit-recognition}
\end{figure*}

\subsection{Decomposed multiplication pipeline}
\label{app:f8-decomposed-curves}

Figure~\ref{fig:f8-appendix-decomposed} shows the training histories of
the two modules used in the decomposed multiplication pipeline. The
matrix vector module is trained on the strictly filtered
transformation set, while the reference basis multiplication module is
trained on the complete multiplication table in the fixed reference
basis. Both modules reach high training exact accuracy across seeds.

\begin{figure*}[t]
    \centering
    \includegraphics[width=\textwidth]
    {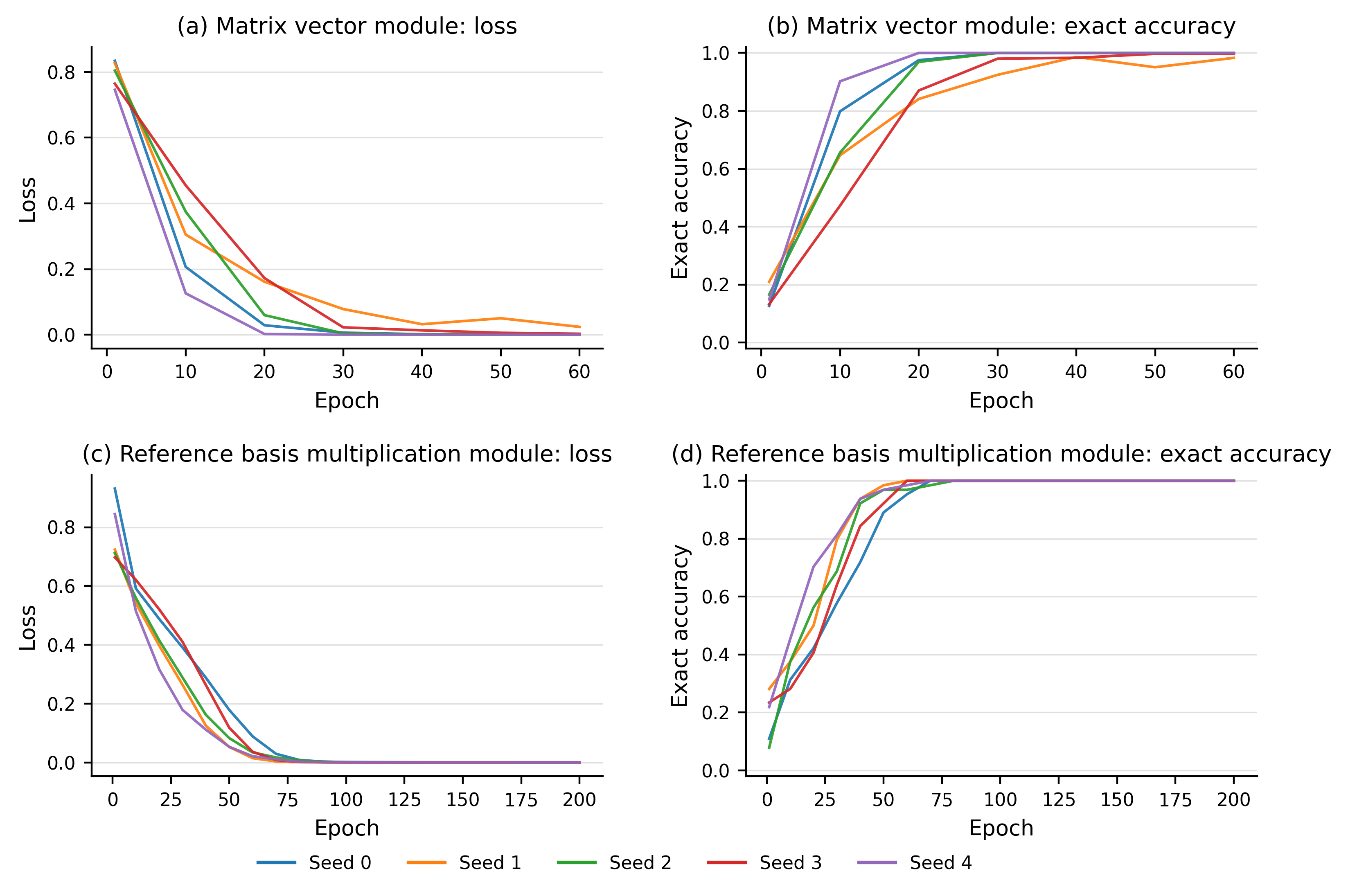}
    \caption{
    Training curves for the two learned modules in the decomposed
multiplication pipeline.
    }
    \label{fig:f8-appendix-decomposed}
\end{figure*}

These curves show the training behavior of the individual modules.
Held-out performance of the complete pipeline under the strict
separation protocol is reported separately in
Table~\ref{tab:decomposed-pipeline-results}.

\subsection{Learning the Galois action}
\label{app:f8-galois-action-curves}
Figure~\ref{fig:f8-appendix-galois-action} shows the training histories
of the learned transformation used to evaluate the Galois action on
held-out predictor inputs. The model reaches perfect training exact
matrix accuracy for all five seeds. These curves describe optimization
on the training transformations. Prediction after one application,
composition after two applications, and cycle closure after three
applications are evaluated separately on held-out inputs. The corresponding results for prediction after one application,
composition after two applications, and cycle closure after three
applications on held-out inputs are reported in
Table~\ref{tab:galois-action-results}.

\begin{figure*}[t]
    \centering
    \includegraphics[width=\textwidth]
    {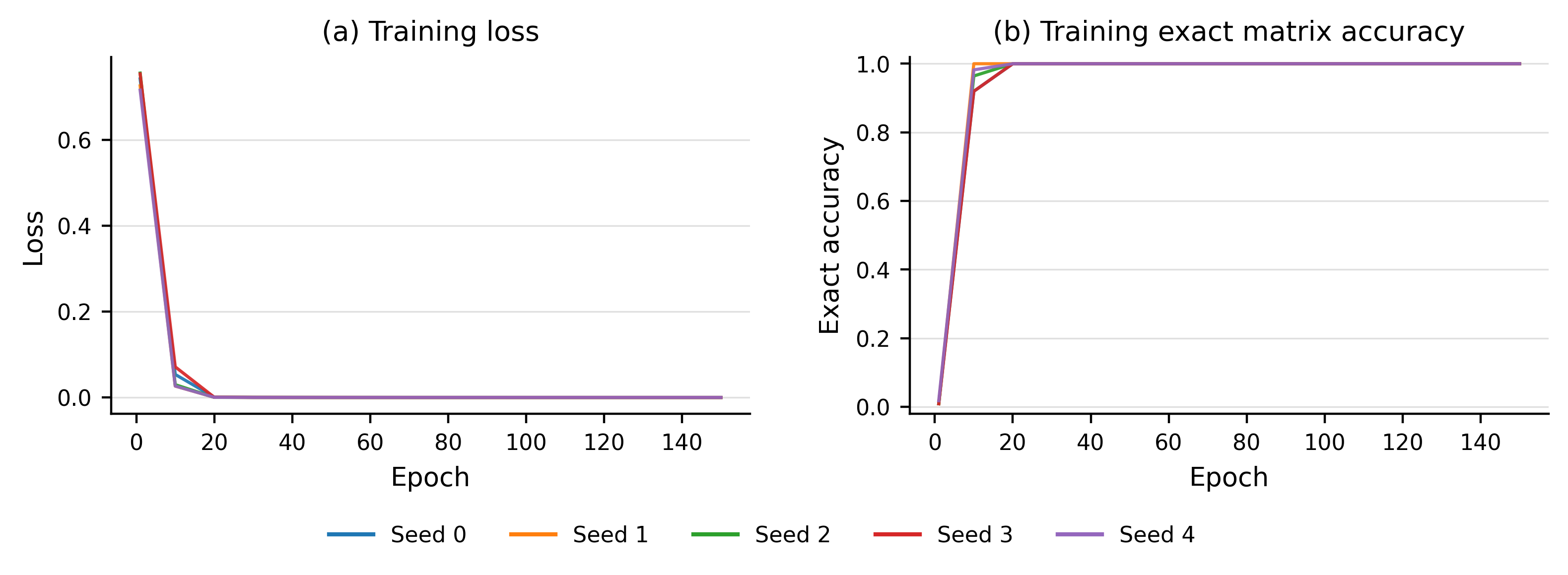}
    \caption{
    Training curves for the learned Galois action. 
    }
    \label{fig:f8-appendix-galois-action}
\end{figure*}

\subsection{Canonicalization and downstream multiplication}
\label{app:f8-canonicalization-curves}

Figure~\ref{fig:f8-appendix-canonicalization} shows the training
histories of the downstream multiplication model used in the
canonicalization pipeline. The model converges to almost perfect
training exact accuracy for all seeds, although its intermediate
trajectories differ across runs.

The learned transformation used to construct the canonical
representatives follows the same training procedure as the learned
Galois action shown in
Figure~\ref{fig:f8-appendix-galois-action}. So its training curves are not repeated here. Exact canonical matching, downstream multiplication under the primary
exact lookup procedure, and separate recovery results based on Hamming
distance are reported in
Table~\ref{tab:canonicalization-results}.

\begin{figure*}[t]
    \centering
    \includegraphics[width=\textwidth]
    {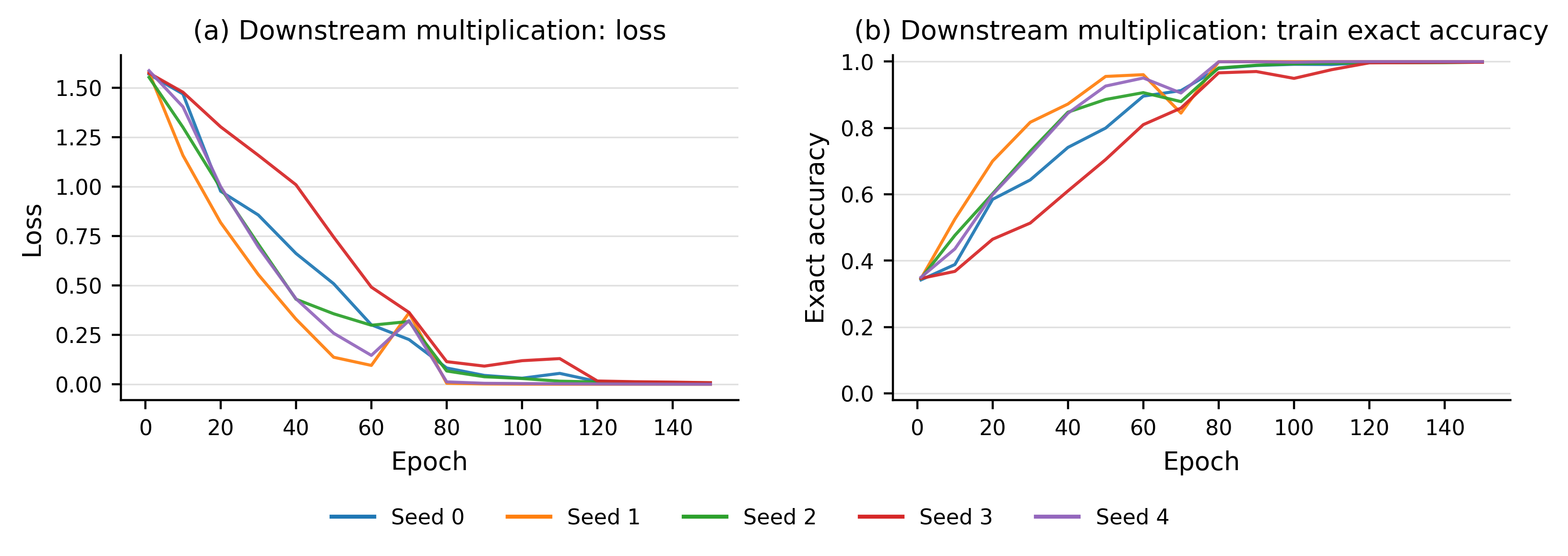}
    \caption{
    Training curves for the downstream multiplication model used in
    the canonicalization pipeline. 
    }
    \label{fig:f8-appendix-canonicalization}
\end{figure*}

\end{document}